\documentclass[a4paper]{article}

\usepackage[a4paper,top=3cm,bottom=2cm,left=3cm,right=3cm,marginparwidth=1.75cm]{geometry}
\usepackage{booktabs} 
\usepackage{amsmath}
\usepackage{mathtools}
\usepackage{amssymb}
\usepackage{bm}
\usepackage{amsthm}
\usepackage{graphicx}
\usepackage{xcolor}
\usepackage[colorinlistoftodos]{todonotes}
\newtheorem{theorem}{Theorem}
\newtheorem{proposition}{Proposition}
\newtheorem{lemma}{Lemma}
\newtheorem{corollary}{Corollary}
\newtheorem{definition}{Definition}
\newtheorem{example}{Example}
\newcommand{\bx}{\mathbf{x}}

\newcommand{\by}{\mathbf{y}}
\newcommand{\bt}{\mathbf{t}}
\newcommand{\be}{\mathbf{e}}

\newcommand{\bv}{\mathbf{v}}
\newcommand{\bw}{\mathbf{w}}
\newcommand{\cost}{c}

\newenvironment{proof-sketch}{%
  \proof}{\endproof}

\newcommand\ignore[1]{}

\DeclareMathOperator*{\argmax}{arg\,max}
\DeclareMathOperator*{\argmin}{arg\,min}
\begin{document}
\title{The Disparate Effects of Strategic Manipulation}
\author{
  Lily Hu\\
  \small{Harvard University}
  \and
  Nicole Immorlica\\
 \small{Microsoft Research}
  \and
  Jennifer Wortman Vaughan\\
  \small{Microsoft Research}
}
\date{November 25, 2018}

\maketitle

\begin{abstract}
When consequential decisions are informed by algorithmic input, individuals may feel compelled to alter their behavior in order to gain a system's approval. Models of agent responsiveness, termed "strategic manipulation," analyze the interaction between a learner and agents in a world where all agents are equally able to manipulate their features in an attempt to ``trick" a published classifier. In cases of real world classification, however, an agent's ability to adapt to an algorithm is not simply a function of her personal interest in receiving a positive classification, but is bound up in a complex web of social factors that affect her ability to pursue certain action responses. In this paper, we adapt models of strategic manipulation to capture  dynamics that may arise in a setting of social inequality wherein candidate groups face different costs to manipulation. We find that whenever one group's costs are higher than the other's, the learner's equilibrium strategy exhibits an inequality-reinforcing phenomenon wherein the learner erroneously admits some members of the advantaged group, while erroneously excluding some members of the disadvantaged group. We also consider the effects of  interventions in which a learner subsidizes members of the disadvantaged group, lowering their costs in order to improve her own classification performance. Here we encounter a paradoxical result: there exist cases in which providing a subsidy improves only the learner's utility while actually making both candidate groups worse-off---even the group receiving the subsidy. Our results reveal the potentially adverse social ramifications of deploying tools that attempt to evaluate an individual's ``quality'' when agents' capacities to adaptively respond differ.
\end{abstract}

\section{Introduction}
The expanding realm of algorithmic decision-making has not only altered the ways that institutions conduct their day-to-day operations, but has also had a profound impact on how individuals interface with these institutions. It has changed the ways we communicate with each other, receive crucial resources, and are granted important social and economic opportunities. Theoretically, algorithms have great potential to reform existing systems to become both more efficient and equitable, but as exposed by various high-profile investigations \cite{sweeney2013discrimination, o2016weapons,angwin2016machine,eubanks2018automating}, prediction-based models that make or assist with consequential decisions are, in practice, highly prone to reproducing past and current patterns of social inequality. 

While few algorithmic systems are explicitly designed to be discriminatory, there are many underlying forces that drive such socially biased outcomes. For one, since most of the features used in these models are based on proxy, rather than causal, variables, outputs often reflect the various structural factors that bear on a person's life opportunities rather than the individualized characteristics that decision-makers often seek. Much of the previous work in algorithmic fairness has examined a particular undesirable proxy effect in which a classifier's features may be linked to socially significant and legally protected attributes like race and gender, interpreting correlations that have arisen due to centuries of accumulated disadvantage as genuine attributes of a particular category of people \cite{johnson2016impartial,qureshi2016causal,kilbertus2017avoiding,grgic2018beyond}. 

But algorithmic models do not only generate outcomes that passively correlate with social advantages or disadvantages. These tools also provoke a certain type of reactivity, in which agents see a classifier as a guide to action and actively change their behavior to accord with the algorithm's preferences. On this view, classifiers both \emph{evaluate} and \emph{animate} their subjects, transforming static data into strategic responses. Just as an algorithm's use of certain features differentially advantages some populations over others, the room for strategic response that is inherent in many automated systems also naturally favors social groups of privilege. Admissions procedures that heavily weight SAT scores motivate students who have the means to take advantage of test prep courses and even take the exam multiple times. Loan approval systems that rely on existing lines of credit as an indication of creditworthiness encourage those who can to apply for more credit in their name. 

Thus an algorithm that scores applicants to determine how a resource should be allocated sets a standard for what an ideal candidate's features ought to be. A responsive subject would look to alter how she appears to a classifier in order to increase her likelihood of gaining the system's approval. But since reactivity typically requires informational and material resources that are not equally accessible to all, even when an algorithm draws on features that seem to arise out of individual effort, these metrics can be skewed to favor those who are more readily able to alter their features. 
 
In the machine learning literature, agent reactivity to a classifier is termed ``strategic manipulation.'' Since previous work in strategic classification has typically depicted agent-classifier interactions as antagonistic, such actions are usually viewed as distortions that aim to undermine a learner's classifier \cite{bruckner2011stackelberg,hardt2016strategic}. As shown in Hardt et al. \cite{hardt2016strategic}, a learner who anticipates these responses can, under certain formulations of agent costs, adapt to protect against the misclassification errors that would have resulted from manipulation, recovering an accuracy level that is arbitrarily close to the theoretical maximum. These results are welcome news for a learner who correctly assesses agents' best-responses. Indeed in most strategic manipulation models, agents are depicted as equally able to pursue manipulation, allowing the learner who knows their costs to accurately preempt strategic responses. While there are occasions in which agents do largely face homogenous costs---an even playing field---in many other social use cases of machine learning tools, agents do not encounter the same costs of altering the attributes that are ultimately observed and assessed by the classifier. As such, in this paper we ask, \emph{``What are the effects of strategic classification and manipulation in a world of social stratification?''}

As in previous work in strategic classification, we cast the problem as a Stackelberg game in which the learner moves first and publishes her classifier before candidates best-respond and manipulate their features \cite{bruckner2011stackelberg,hardt2016strategic,akyol2016price, datta2017proxy}. But in contrast with the models in Br{\"u}ckner \& Scheffer \cite{bruckner2011stackelberg} and Hardt et al. \cite{hardt2016strategic}, we formalize the setting of a society comprised of social groups that not only may differ in terms of distributions over unmanipulated features and true labeling functions but also face different costs to manipulation. This extra set of differences brings to light questions that favor an analysis that focuses on the welfares of the candidates who must contend with these classifiers: Do classifiers formulated with strategic behavior in mind impose disparate burdens on different groups? If so, how can a learner mitigate these adverse effects? The altered gameplay and outcomes of strategic classification beg questions of fairness that are intertwined with those of optimality.  

Though our model is quite general, we obtain technical results that reveal important social ramifications of using classification in systems marked by deep inequalities and a potential for manipulation. Our analysis shows that, under our model, even when the learner knows the costs faced by different groups, her equilibrium classifier will always act to reinforce existing inequalities by mistakenly excluding qualified candidates who are less able to manipulate their features while also mistakenly admitting those candidates for whom manipulation is less costly, perpetuating the relative advantage of the privileged group. We delve into the cost disparities that generate such inevitable classification errors.

Next, we consider the impact of providing subsidies to lighten the burden of manipulation for the disadvantaged group. We find that such an intervention can improve the learner's classification performance as well as mitigate the extent to which her errors are inequality-reinforcing. However, we show that there exist cases in which providing subsidies enforces an equilibrium learner strategy that actually makes some individual candidates worse-off without making any better-off. Paradoxically, in these cases, paying a subsidy to the disadvantaged group actually benefits only the learner while both candidate groups experience a welfare decline! Further analysis of these scenarios reveals that, in many cases, all parties would have preferred a world in which manipulation of features was not possible for any candidates. 

Our paper's agent-centric analysis views data points as representing individuals and classifications as impacting those individuals' welfares. This orientation departs from the dominant perspective in learning theory, which privileges a vendor's predictive accuracy, and instead evaluates classification regimes in light of the social consequences of the outcomes they issue. By incorporating insights and techniques from game theory and economics, domains that consider deeply the effects of various policies on agents' behaviors and outcomes, we hope to broaden the perspective that machine learning takes on socially-oriented tools. Presenting more democratically-inclined analysis has been central to the field of algorithmic fairness, and we hope our work sheds new light on this generic setting of classification with strategic agents. 

\subsection{Related Work}
While many earlier approaches to strategic classification in the machine learning literature have tended to view learner-agent interactions as adversarial \cite{kearns1993learning,auer1998line}, our work does not assume inherently antagonistic relationships, and instead, shares the Stackelberg game-theoretic perspective akin to that presented in Br{\"u}ckner \& Scheffer \cite{bruckner2011stackelberg} and built upon by Hardt et al. \cite{hardt2016strategic}. Departing from these models' focus on static prediction and homogeneous manipulation costs, Dong et al. \cite{dong2018strategic} propose an online setting of strategic classification in which agents appear sequentially and have individual costs for manipulation that are unknown to the learner. Unlike our work, they take a traditional learner-centric view, whereas our concerns are with the welfare of the candidates. 

Agent features and potential manipulations in the face of a learner classifier can also be interpreted as serving \textit{informational} purposes. In the economics literature on signaling theory, agents interact with a principal---the counterpart to our learner---via signals that convey important information relevant to a particular task at hand. Classic works, such as Spence's paper on job-market signaling, focus their analysis on the varying quality of information that signals provide at equilibrium \cite{spence1978job}. The emphasis in our analysis on different group costs shares features with a recent update to the signaling literature by Frankel \& Kartik \cite{frankel2014muddled}, who also distinguish between natural actions, corresponding to unmanipulated features in our model, and ``gaming" ability, which operate similarly to our cost functions. The connection between gaming capacity and social advantage is also explicitly discussed in work by Esteban \& Ray \cite{esteban2006inequality} who consider the effects of wealth and lobbying on governmental resource allocation. While most works in the economics signaling literature  center on the decay of the informativeness of signals as gaming and natural actions become indistinguishable, some recent work in computer science has also considered the effect of costly signaling on mechanism design \cite{kephart2015complexity,kephart2016revelation}. In contrast to both of these perspectives, our work highlights the effect of manipulation on a learner's action and as a consequence, on the agents' welfares.

In independent, concurrent work appearing at the same conference, Milli et al. \cite{milli2018social} also consider the social impacts of strategic classification. Whereas our model highlights the interplay between a learner's Stackelberg equilibrium classifier and agents' best-response manipulations at the feature level, their work traces the relationship between the learner's utility and the social burden, a measure of agents' manipulation costs. They show that an institution must select a point on the outcome curve that trades off its predictive accuracy with the social burden it imposes. In their model, an agent with an unmanipulated feature vector $\bx$ has a likelihood $\ell(\bx)$ of having a positive label and can manipulate to any vector $\by$ with $\ell(\by) \le \ell(\bx)$ at zero cost, or to $\by$ with $\ell(\by) > \ell(\bx)$ for a positive cost. This assumption, called ``outcome monotonicity," allows them to reason about manipulations in (one-dimensional) likelihood space rather than feature space, since the optimal learner strategies amount to thresholds on likelihoods. In contrast, we allow features to be differently manipulable (perhaps a student can boost her SAT score via test prep courses, but can do nothing to change her grades from the previous year, and cannot freely obtain a higher SAT score in exchange for a worse record of extracurricular activities), which affects the forms of both the learner's equilibrium classifier and agents' best-response manipulations. Despite these differences in model and focus, their analysis yields results that are qualitatively similar to ours. Highlighting the differential impact of classifiers on social groups, they also find that overcoming stringent thresholds is more burdensome on the disadvantaged group.

\section{Model Formalization}
As in Br{\"u}ckner \& Scheffer \cite{bruckner2011stackelberg} and Hardt et al. \cite{hardt2016strategic}, we formalize the Strategic Classification Game as a Stackelberg competition in which the learner moves first by committing to and publishing a binary classifier $f$. Candidates, who are endowed with ``innate'' features, best respond by manipulating their feature inputs into the classifier. Formally, a candidate is defined by her $d$-dimensional feature vector $\bx \in X = [0,1]^d$ and group membership $A$ or $B$, with $A$ signifying the advantaged group and $B$ the disadvantaged. Group membership bears on manipulation costs such that a candidate from group $m$ who wishes to move from a feature vector $\bx$ to a feature vector $\by$ must pay a cost of $\cost_m(\by) - \cost_m(\bx)$. We note that these cost function forms are similar to the class of separable cost functions considered in Hardt et al. \cite{hardt2016strategic}. We assume that higher feature values indicate higher quality to the learner, and thus restrict our attention to manipulations such that $\by \ge \bx$, where the symbol $\ge$ signifies a component-wise comparison such that $\by \ge \bx$ if and only if $\forall i \in [d]$, $y_i \ge x_i$. Throughout this paper, we study non-negative monotone cost functions such that the cost of manipulating from a feature vector $\bx$ to a feature vector $\by$ increases as $\bx$ and $\by$ get further apart.

To motivate this distinction between features and costs, consider the use of SAT scores as a signal of academic preparedness in the U.S. college admissions process. The high-stakes nature of the SAT has encouraged the growth of a test prep industry dedicated to helping students perform better on the exam. Test preparation books and courses, while also exposing students to content knowledge and skills that are covered on the SAT, promise to ``hack" the exam by training students to internalize test-taking strategies based on the format, structure, and style of its questions. One can view SAT scores as a feature used by a learner building a classifier to select candidates with sufficient academic success according to some chosen standard. The existence of test prep resources then presents an opportunity for some applicants to inflate their scores, which might ``trick'' the tool into classifying the candidates as more highly qualified than they are in actuality. In this example, a candidate's strategic manipulation move refers to her investment in these resources, which despite improving her exam score, do not confer any genuine benefits to her level of academic preparation for college.  

\ignore{
In a society with structural inequality, feature manipulation for the purpose of pursuing some desired classification outcome is not equally accessible to all. As such, candidate groups signify memberships that confer relative advantaged or disadvantaged status.
}

Just as access to test prep resources tends to fall along income and race lines, we view candidates' different abilities to manipulate as tied to their group membership.  We model these group differences with respect to availability of resources and opportunity by enforcing a \textit{cost condition} that orders the two groups. We suppose that for all $\bx \in [0,1]^d$ and $\by \ge \bx$, \begin{equation} \label{cost-condition}
 c_A(\by) - c_A(\bx) \le c_B(\by) - c_B(\bx).
 \end{equation}
Manipulating from a feature vector $\bx$ to $\by$ is always at least as costly for a member of group $B$ as it is for a member of group $A$. We believe our model's inclusion of this cost condition reflects an authentic aspect of our social world wherein one group is systematically disadvantaged with respect to a task in comparison to another.

In our setup, we also allow groups to have distinct probability distributions $\mathcal{D}_A$ and $\mathcal{D}_B$ over unmanipulated features and to be subject to different true labeling functions $ h_A$ and $h_B$ defined as
\begin{equation}
h_A(\bx) = \begin{cases} 1, & \forall \bx \text{ such that }  \sum_{i=1}^dw_{A, i} x_i \ge \tau_A , \\ 0, & \forall \bx \text{ such that }  \sum_{i=1}^dw_{A, i} x_i < \tau_A , \end{cases} 
\end{equation}
\begin{equation}
h_B(\bx) = \begin{cases} 1, & \forall \bx \text{ such that }  \sum_{i=1}^dw_{B, i} x_i \ge \tau_B , \\ 0, & \forall \bx \text{ such that }  \sum_{i=1}^dw_{B, i} x_i < \tau_B . \end{cases} 
\end{equation}
We assume that $h_A(\bx) = 1 \implies h_B(\bx) = 1$ for all $\bx \in [0, 1]$. Returning to the SAT example, research has shown that scores are skewed by race even before factoring in additional considerations such as access to manipulation \cite{card2007racial}. In such cases, the true threshold for the disadvantaged group is lower than that for the advantaged group. We leave this generality in our model to acknowledge and account for the influence that various social and historical factors have on candidates' unmanipulated features and not, we emphasize, as an endorsement of a view that groups are fundamentally different in ability. A formal description of the Strategic Classification Game with Groups is given in the following definition. 

\begin{definition}[Strategic Classification Game with Groups]
In the Strategic Classification Game with Groups, candidates with features $\bx \in [0, 1]^d$ and group memberships $A$ or $B$ are drawn from distributions $\mathcal{D}_A$ and $\mathcal{D}_B$. The population proportion of each group is given by $p_A$ and $p_B$ where $p_A + p_B = 1$. A candidate from group $m$ pays cost $c_m(\by) - c_m(\bx)$ to move from her original features $\bx$ to $\by \ge \bx$. There exist true binary classifiers $h_A$ and $h_B$, for candidates of each group. Probability distributions, cost functions, and true binary classifiers are all common knowledge. Gameplay proceeds in the following manner:
\begin{enumerate}
\item The learner issues a classifier $f$ generating outcomes $\{0, 1\}$. 
\item Each candidate observes $f$ and manipulates her features $\bx$ to $\by \ge \bx$.
\end{enumerate}
A group $m$ candidate with features $\bx$ who moves to $\by$ earns a payoff \[f(\by) - (c_m(\by) - c_m(\bx)).\]
The learner incurs a penalty of
\begin{equation*}
\begin{split}
&C_{FP} \sum_{m \in \{A, B\}} p_m P_{\bx\sim\mathcal{D}_m}[h_m(\bx) = 0, f(\by) = 1] \\
& + C_{FN}\sum_{m \in \{A, B\}} p_m P_{\bx\sim\mathcal{D}_m}[h_m(\bx) = 1, f(\by) = 0] ,
\end{split}
\end{equation*}
where $C_{FP}$ and $C_{FN}$ denote the cost of a false positive and a false negative respectively.
\end{definition}

The learner looks to correctly classify candidates with respect to their original features $\bx$, whereas each candidate hopes to manipulate her features to attain a positive classification, expending as little cost as possible in the process. Under this setup, candidates are only willing to manipulate their features if it flips their classification from $0$ to $1$ and if the cost of the manipulation is less than $1$. We note that defining the utility of a positive classification to be $1$ can be considered a scaling and thus is without loss of generality.

This learner-candidate interaction is very similar to that studied in Hardt et al. \cite{hardt2016strategic}. However, our inclusion of groups with distinct manipulation costs leads to an ambiguity regarding a candidate's initial features that does not exist when all candidates have an equal opportunity to manipulate. In very few cases can a vendor distinguish among candidates based on their group membership for the explicit purpose of issuing distinct classification policies, especially if that group category is a protected class attribute. As such, in our setup, we require that a learner publish a classifier that is not adaptive to different agents based on their group identities. 


It is important to note that the positive results in Hardt et al.'s \cite{hardt2016strategic} formulation of the Strategic Classification Game, wherein for separable cost functions, the learner can attain a classification error at test-time that is arbitrarily close to the optimal payoff attainable, do not carry over into this setting of heterogeneous groups and costs. Even when $h_A = h_B$, the existence of different costs of agent manipulation, even when separable as in our model, introduces a base uncertainty to the learning problem that generates errors that cannot be extricated so long as the learner must publish a classifier that does not distinguish candidates based on their group memberships. Second, an analysis of the learner's strategy and performance, the  perspective typically taken in most learning theory papers, contributes only a partial view of the total welfare effect of using classification in strategic settings. The main objective of this paper is to offer a more thorough and holistic inspection of all agents' outcomes, paying special heed to the different outcomes experienced by candidates of the two groups. Insofar as all social behaviors are impelled by goals, interests, and purposes, we should view data that is strategically generated to be the rule rather than the exception in social machine learning settings.

\subsubsection*{Remark on the assumption that $h_A$ and $h_B$ are known.}
Our assumption that the learner has knowledge of groups' true labeling functions is not central to our analysis. We make such an assumption to highlight the pure effect of groups' differential costs of manipulation on equilibrium gameplay and consequent welfares rather than the potential side effects due to a learner's noisy estimation of the true classifiers. Our general findings do not substantially rely on this feature of the model, and the overall results carry through into a setting in which the learner optimizes from samples. 
\subsubsection*{Remark on unequal group costs}
The differences in costs $c_A$ and $c_B$ encoded by the cost condition is not restricted to referring only to differences in the monetary cost of manipulation. Instead, as is common in information economics and especially signaling theory, ``cost'' reflects the multiplicity of factors that bear on the effort exertion required by feature manipulation \cite{spence1978job,spremann1987agent,laffont2009theory,ballwieser2012agency}. To demonstrate the generality of our formulation of distinct group costs, we show that the cost condition given in (\ref{cost-condition}) is equivalent to a more explicit derivation of the choice that an agent faces when deciding whether to manipulate her feature. 

A rational agent with feature $\bx$ will only pursue manipulation if her value for a positive classification minus her cost of manipulation exceeds her value for a negative classification:
\begin{equation}
\label{explicit-costs}
v(f(\bx) = 0) \le v(f(\by) = 1) - u(c(\by) - c(\bx)) .
\end{equation}
The monotone function $u$ translates the costs borne by a candidate to manipulate from $\bx$ to $\by$ into her ``utility space," i.e., it reflects the value that she places on that expenditure. We can rewrite the previous inequality to be
\begin{equation}
\label{explicit-costs2}
c(\by) - c(\bx) \le u^{-1}\big(v(f(\by) = 1)- v(f(\bx) = 0)\big) .
\end{equation}
Substituting in $k = u^{-1}\big(v(f(\by) = 1)- v(f(\bx) = 0)\big)$, we have  
$c(\by) - c(\bx) \le k.$
Since the same cost expenditure is valued more highly by the disadvantaged group than by the advantaged group, the function $u$ is more convex for group $B$ than for group $A$. Thus all else equal, we have $c_A(\by) - c_A(\bx) \le c_B(\by) - c_B(\bx)$ as desired. More generally, the functions $v$, $c$, and $u$ may each be different for the groups. As such, the disadvantage encoded in the cost condition can arise due to differences in valuations of classifications ($v$), differences in costs ($c$), or differences in valuations of those costs ($u$).

\section{Equilibrium Analysis}
\label{sec:eq}
We begin by studying agents' best-response strategies in the basic Strategic Manipulation Game with Groups in which candidates belong to one of two groups $A$ and $B$, and the cost condition holds so that group $B$ members face greater costs to manipulation than group $A$ members. 
\ignore{
Whereas a candidate's feature $\bx$ does capture some relevant metric of interest for the learner, her manipulation costs do not provide any additional indication of the quality of her candidacy. Therefore a difference in cost only reflects a difference in opportunity to manipulate one's feature as caused by external factors. 
}
To build intuition, we first consider best-response strategies in the one-dimensional case in which candidates have features $x \in [0,1]$ and group cost functions are of any non-negative monotone form. We then move on to consider the $d$-dimensional case in which candidate features are given as vectors $\bx \in [0,1]^d$ and manipulation costs are assumed to be linear. 

\subsection{One-dimensional Features}
\label{sec:eq1d}
In the $d=1$ case, the cost condition given in (\ref{cost-condition}) may be written as
$c_A'(x) \le c_B'(x)$
for all $x \in [0,1]$. Since the true decision boundaries are linear, in the one-dimensional case, they may be written as threshold functions where thresholds $\tau_A$ and $\tau_B$ are constants in $[0,1]$ and for agents in group $m$, $h_m(x) = 1$ if and only if $x \geq \tau_m$. A university admissions decision based on a single score is an example of such a classifier. Although the SAT does not act as the sole determinant of admissions in the U.S., in countries such as Australia, Brazil, and China, a single exam score is often the only factor of applicant quality that is considered for admissions. 

When the learner has access to $\tau_A$ and $\tau_B$, and group costs $c_A$ and $c_B$ satisfy the cost condition, the following proposition characterizes the space of undominated strategies for the learner who seeks to minimize any error-penalizing cost function.

\begin{proposition}[One-D Undominated Learner Strategies]
\label{1d-optimal}
Given group cost functions $c_A$ and $c_B$ and true label thresholds $\tau_A$ and $\tau_B$ where $\tau_B \le \tau_A$, there exists a space of undominated learner threshold strategies $[\sigma_B, \sigma_A] \subset [0, 1]$ where $\sigma_A = c_A^{-1}(c_A(\tau_A)+1)$ and $\sigma_B = c_B^{-1}(c_B(\tau_B)+1)$. That is, for any error penalties $C_{FP}$ and $C_{FN}$, the learner's equilibrium classifier $f$ is based on a threshold $\sigma \in [\sigma_B, \sigma_A]$ such that for all manipulated features $y$,
\begin{equation}
f(y) = \begin{cases} 1, & \forall y \ge \sigma, \\ 0, & \forall y < \sigma. \end{cases} 
\end{equation}
\end{proposition}

To understand this result, first notice that if the learner were to face only those candidates from group $A$, she would achieve perfect classification by labeling as $1$ only those candidates with unmanipulated feature $x \ge \tau_A$. This strategy is enacted by considering candidates' best-response manipulations. A rational candidate would only be willing to manipulate her feature if the gain she receives in her classification exceeds her costs of manipulation. The learner would like to guard against manipulations by candidates with $x < \tau_A$ but still admit candidates with $x \ge \tau_A$, so she considers the maximum manipulated feature $y$ that is attainable by a rational candidate with $x = \tau_A$ who is willing to spend up to a cost of one in order to secure a better classification, as illustrated in Figure~\ref{fig:1d}. The maximum such $y$ value is $\sigma_A$, and thus, the learner sets a threshold at $\sigma_A$, admitting all those with $y \ge \sigma_A$ and rejecting all those with $y < \sigma_A$. The same reasoning applies to a learner facing only group $B$ candidates, and the learner sets a threshold at $\sigma_B$, admitting all those candidates with $y \ge \sigma_B$ and rejecting all those with $y < \sigma_B$.

It can be shown that for all valid values of $\tau_A, \tau_B, c_A,$ and $c_B$, necessarily $\sigma_B \le \sigma_A$. Then all classifiers with threshold $\sigma < \sigma_B$ are dominated by $\sigma_B$, in the sense that for any arbitrary error penalties $C_{FP}$ and $C_{FN}$, the learner would suffer higher costs by setting her threshold to be $\sigma$ rather than $\sigma_B$. In the same way, all thresholds $\sigma > \sigma_A$ are dominated by $\sigma_A$, thus leaving $[\sigma_B, \sigma_A]$ to be the space of undominated thresholds. For an account of the full proof of this result (and all omitted proofs), see the appendix.  

Even without committing to a particular learner cost function, the space of optimal strategies characterized in Proposition \ref{1d-optimal} leads to an important consequence. A rational learner in the Strategic Classification Game always selects a classifier that exhibits the following phenomenon: it mistakenly admits unqualified candidates from the group with lower costs and mistakenly excludes qualified candidates from the group with higher costs. This result is formalized in Proposition \ref{1d-cost}.  

To state the proposition, the following definition is instructive. Whereas the true thresholds $\tau_A$ and $\tau_B$ are a function of unmanipulated features, the learner only faces candidate features that may have been manipulated. In order to make these observed features commensurable with $\tau_A$ and $\tau_B$, it is helpful for the learner to ``translate'' a candidate's possibly manipulated feature $y$ to its minimum corresponding original unmanipulated value. 

\begin{definition}[Correspondence with unmanipulated features]
For any observed candidate feature $y \in [0,1]$, the minimum corresponding unmanipulated feature is defined as
\begin{equation} 
\begin{split}
\label{correspondence}
\ell_A(y)= \max\{0, c_A^{-1}(c_A(y)-1)\},
\\
\ell_B(y)= \max\{0, c_B^{-1}(c_B(y)-1)\}
\end{split}
\end{equation}
for a candidate belonging to group A and group B respectively.
\end{definition}
The corresponding values $\ell_A(y)$ and $\ell_B(y)$ are defined such that a candidate who presents feature $y$ must have as her true unmanipulated feature $x \ge \ell_A(y)$ if she is a group A member and $x \ge \ell_B(y)$ if she is a group B member. 

\begin{figure}
\centering
\includegraphics[width = 7cm]{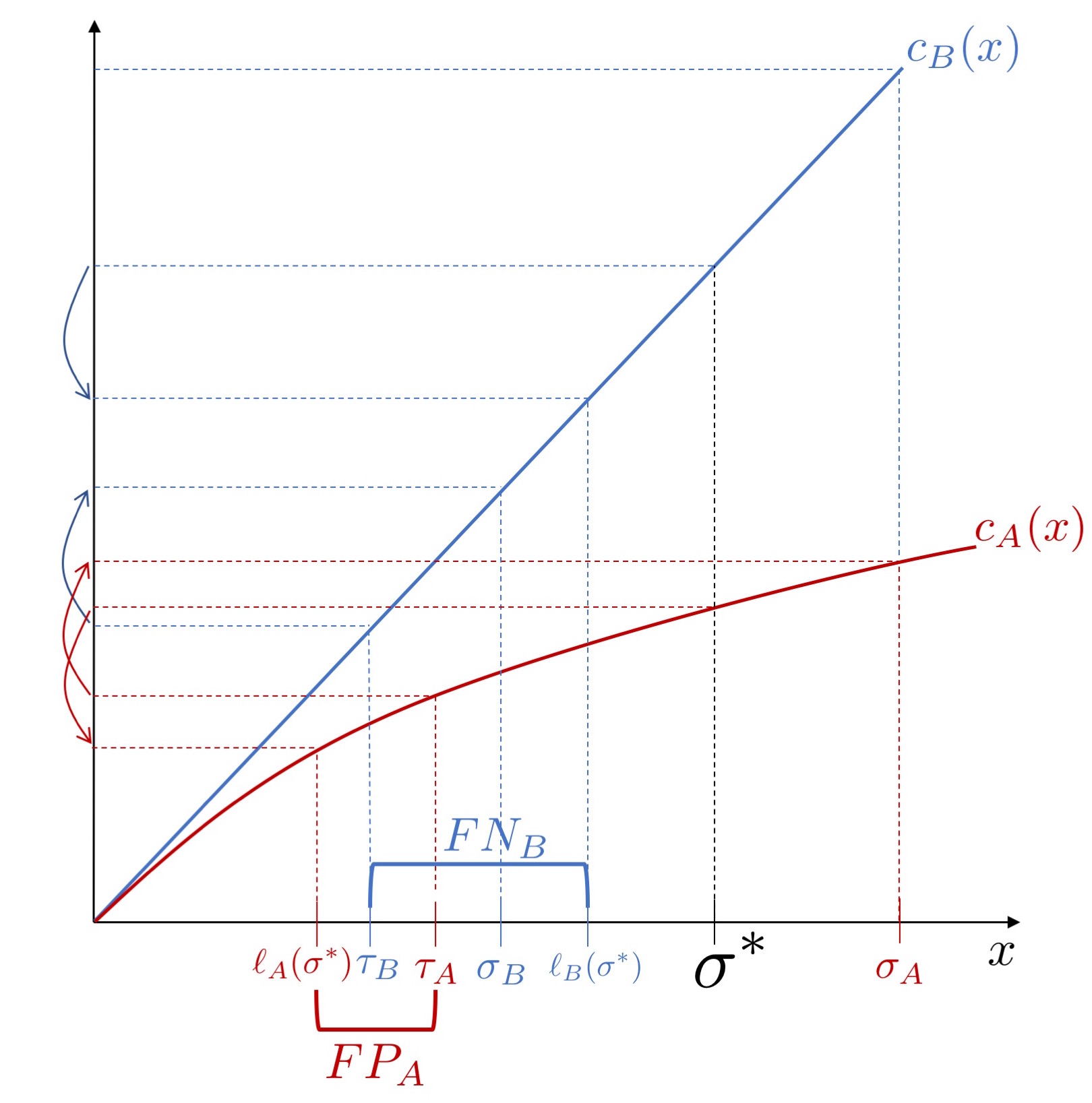}
\caption{Group cost functions for a one-dimensional feature $x$. $\tau_A$ and $\tau_B$ signify true thresholds on unmanipulated features for group $A$ and $B$, but a learner must issue a classifier on manipulated features. The threshold $\sigma_A$ perfectly classifies group $A$ candidates; $\sigma_B$ perfectly classifies group $B$ candidates. A learner selects an equilibrium threshold $\sigma^* \in [\sigma_B, \sigma_A]$, committing false positives on group $A$ (red bracket) and false negatives on group $B$ (blue bracket).\label{fig:1d}}
\end{figure}
\begin{proposition}[Learner's Cost in 1 Dimension]
\label{1d-cost}
A learner who employs a classifier $f$ based on a threshold strategy $\sigma \in [\sigma_B, \sigma_A]$ only commits false positives errors on group A and false negatives errors on group B. The cost $C(\sigma)$ of such a classifier is 
\[ C_{FN} p_B P_{x\sim\mathcal{D}_B}\big[x \in [\tau_B, \ell_B(\sigma))\big] + C_{FP} p_A P_{x\sim\mathcal{D}_A} \big[x \in [\ell_A(\sigma), \tau_A)\big], \]
where false negative errors entail penalty $C_{FN}$, and false positive errors entail penalty $C_{FP}$.
\end{proposition}

A learner who commits to classifying only one of the groups correctly bears costs given by the following corollaries.

\begin{corollary}
\label{corollary-fn}
A classifier based on $\sigma_A$ perfectly classifies group A candidates and bears cost $C(\sigma_A) = C_{FN}p_BP_{x\sim\mathcal{D}_B}\big[x \in [\tau_B, \ell_B(\sigma))\big].$
\end{corollary}
\begin{corollary}
\label{corollary-fp}
A classifier based on $\sigma_B$ perfectly classifies group B candidates and bears cost $C(\sigma_B) = C_{FP} p_AP_{x\sim\mathcal{D}_A} \big[x \in [\ell_A(\sigma), \tau_A)\big].$
\end{corollary}

Notice that the learner's errors always cut in the same direction---by unduly benefiting group $A$ candidates and unduly rejecting group $B$ candidates, these errors act to reinforce the existing social inequality that had generated the unequal group cost conditions in the first place. Since these errors arise out of the asymmetric group costs of manipulation, the Strategic Classification Game can be viewed as an interactive model that itself perpetuates the relative advantage of group A over group B candidates.

Within the undominated region $[\sigma_B, \sigma_A]$, the equilibrium learner threshold $\sigma^*$ is attained as the solution to the optimization problem
\begin{equation}
\sigma^* = \argmin_{\sigma \in [\sigma_B, \sigma_A]} C(\sigma).
\end{equation}
In the game's greatest generality where candidates are drawn from arbitrary probability distributions, groups bear any costs that abide by the cost condition, and the learner has arbitrary error penalties, one cannot specify the equilibrium learner threshold $\sigma^*$ any further. However, under some special cases of candidate cost functions and probability distributions, the equilibrium threshold can be characterized more precisely. Specifically, when candidates from both groups are assumed to be drawn from a uniform distribution over unmanipulated features in $[0,1]$, an error-minimizing learner seeks a threshold value $\sigma^*$ that minimizes the length of the interval of errors, given by the following quantity: 
\[ \sigma^* =  \argmin_{\sigma \in[\sigma_B, \sigma_A]} \ell_B(\sigma) - \ell_A(\sigma) .\]
From here, one natural assumption of candidate cost functions would have that groups $A$ and $B$ bear costs that are proportional to each other. In this case, the curvature of the cost functions is determinative of a learner's equilibrium threshold.

\begin{proposition}
\label{proportional-costs}
Suppose group cost functions are proportional such that $c_A(x) = q c_B(x)$ for $q \in (0,1)$, that $\mathcal{D}_A$ and $\mathcal{D}_B$ are uniform on $[0,1]$, and that $C_{FN} = C_{FP}$ and $p_A = p_B = \frac{1}{2}$. Let $\sigma^*$ be the learner's equilibrium threshold.
\\
When cost functions are strictly concave, $\sigma^* = \sigma_B$. When cost functions are strictly convex, $\sigma^* = \sigma_A$. When cost functions are affine, the learner is indifferent between all $\sigma^* \in [\sigma_B, \sigma_A]$.
\end{proposition}

\subsection{General $d$-Dimensional Feature Vectors}
\label{sec:eqdd}
In the general $d$-dimensional case of the Strategic Classification Game, candidates are endowed with features that are given by a vector $\bx \in [0,1]^d$ and can choose to manipulate and present any feature $\by \ge \bx$ to the learner. In this section, we consider optimal learner and candidate strategies when group costs are linear such that they may be written as
\begin{equation}
c_A(\bx) = \sum_{i=1}^d c_{A, i} x_i;\hspace{10pt} c_B(\bx) = \sum_{i=1}^d c_{B, i} x_i
\end{equation}
for groups $A$ and $B$ respectively. 
Now, the cost condition $c_A(\by) - c_A(\bx) \le c_B(\by) - c_B(\bx)$ for all $\by \ge \bx$---defined component-wise as before---implies that $\forall i \in [d]$, $c_{A,i} \le c_{B,i}$. In $d$ dimensions, the true classifiers $h_A$ and $h_B$ have linear decision boundaries
such that for a group $A$ candidate with feature $\bx$,
\begin{equation}
h_A(\bx) = \begin{cases}
1 & \sum_{i=1}^d w_{A, i}x_i \ge \tau_A ,\\
0 & \sum_{i=1}^d w_{A, i}x_i <\tau_A ,
\end{cases}
\end{equation}
and for a group $B$ candidate with feature $\bx$, 
\begin{equation}
h_B(\bx) = \begin{cases}
1 & \sum_{i=1}^d w_{B, i}x_i \ge \tau_B ,\\
0 & \sum_{i=1}^d w_{B, i}x_i <\tau_B.
\end{cases}
\end{equation}
We assume that all components $x_i$ contribute positively to an agent's likelihood of being classified as $1$ so that $w_{A, i}, w_{B, i} \ge 0$ for all $i$. To ensure that the cost of manipulation is always non-negative, all cost coefficients are positive: $c_{B,i}, c_{A,i} \ge 0$ for all $i \in [d]$.
 
A candidate may now manipulate any combination of the $d$ components of her initial feature $\bx$ to reach the final feature $\by$ that she presents to the learner. Despite this increased flexibility on the part of the candidate, we are still able to characterize the performance of undominated learner classifiers, generalizing the result in Proposition \ref{1d-cost}. All potentially optimal classifiers exhibit the same inequality-reinforcing property inherent within the one-dimensional interval of undominated threshold strategies, trading off false positives on group A candidates with false negatives on group B candidates. Before we formally present this result, we first describe candidates' best-response strategies. Here, a geometric view of the space of potential manipulations is informative. 

Suppose a candidate endowed with a feature vector $\bx$ faces costs $\sum_{i=1}^d c_i x_i$ and is willing to expend a total cost of $1$ for manipulation. Then she can move to any $\by \ge \bx$ contained within the $d$-simplex with orthogonal corner at $\bx$ and remaining vertices at $\bx + \frac{1}{c_i}\be_i$ where $\be_i$ is the $i$th standard basis vector. This region is given by
\begin{equation}
\label{simplex-forward}
\Delta(\bx)= \Big\{ \bx + \sum_{i=1}^d \frac{t_i}{c_i}\be_i  \in [0,1]^d  \Big| \sum_{i=1}^d t_i \le 1\text{ ; } t_i \ge 0 \text{ }\forall i \Big\} .
\end{equation}
 $\Delta(\bx)$, depicted in Figure \ref{f-simplex}, gives the space of potential movement for a candidate with unmanipulated feature $\bx$ who is willing to expend a total cost of $1$. Notice that $t_i$ can be interpreted as the cost that a candidate expends on movement in the $i$th direction. Thus $\sum_{i=1}^d t_i$ gives the total cost of manipulation. Moving beyond the range of possible moves, in order to describe how a rational candidate will best-respond to a learner, we must consider the published classifier. 

\begin{figure}[t!]
\centering
\includegraphics[width=4.5cm]{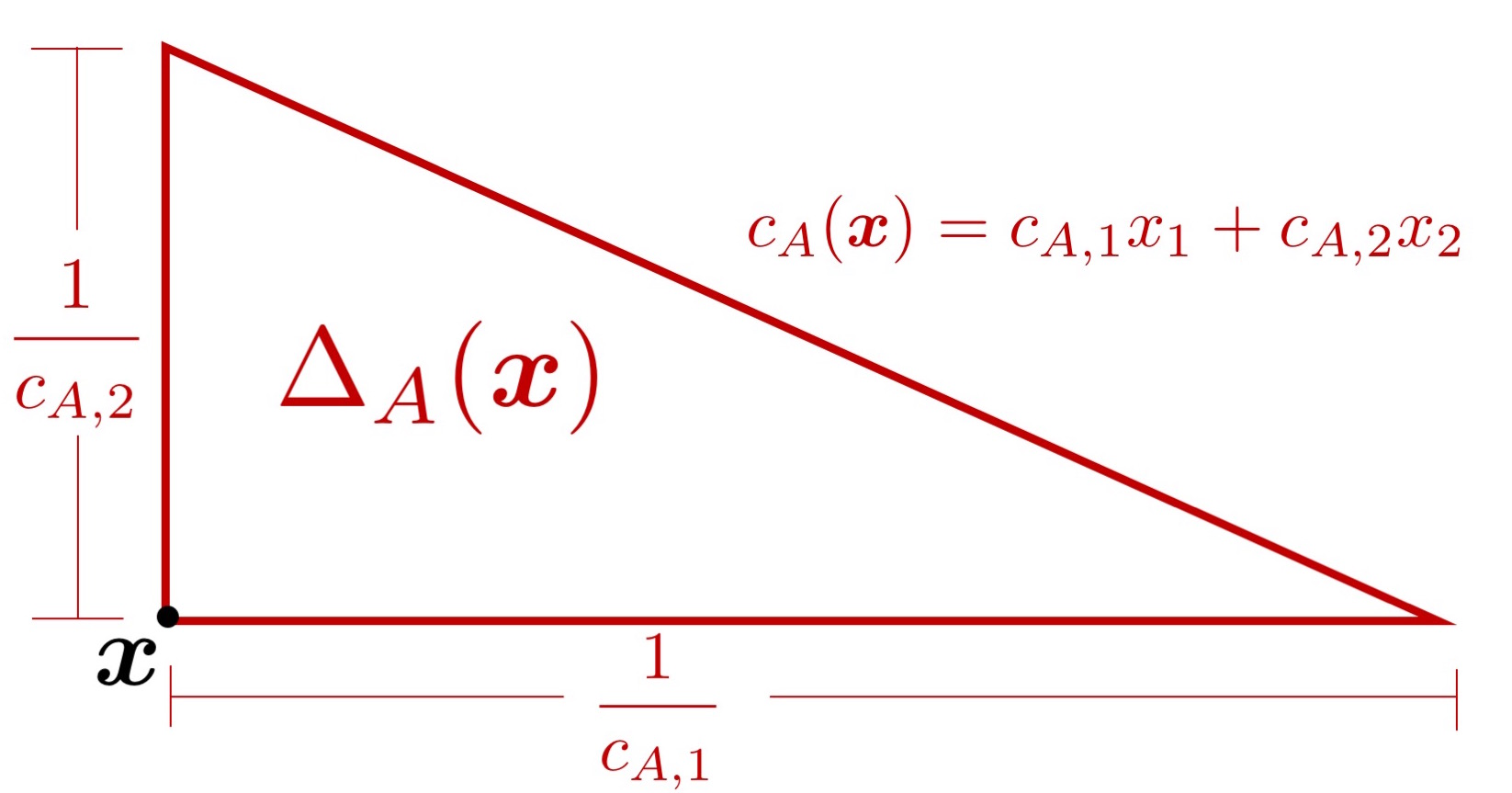}
\caption{The forward simplex. A candidate in group $A$ with unmanipulated feature vector $\bx$ can manipulate to reach any feature vector $\by \in \Delta_A(\bx)$ at a cost of at most 1.}
\label{f-simplex}
\end{figure}

Suppose a learner publishes a classifier $f$ based on a hyperplane $\sum_{i=1}^d g_i y_i = g_0$, so that $f(\by) = 1$ if and only if $\sum_{i=1}^d g_i y_i \geq g_0$. A best-response manipulation occurs along the direction that generates the greatest increase in the value $\sum_{i=1}^d g_i (y_i - x_i)$ for the least cost. As such, a candidate will move in any directions $i \in \argmax_{i \in [d]} \frac{g_i}{c_i}$. This result is formalized in the following lemma. 
 
\begin{lemma}[$d$-D Candidate Best Response]
\label{lemma:opt_candidate}
Suppose a learner publishes the classifier $f(\by) = 1$ if and only if $\sum_{i=1}^d g_i y_i \geq g_0$.  Consider a candidate with unmanipulated feature vector $\bx$ and linear costs $\sum_{i=1}^d c_i x_i$. If $f(\bx) = 1$ or if for all $i \in [d]$, $f(\bx + \frac{1}{c_i} \be_i) = 0$, the candidate's best response is to set $\by = \bx$. Otherwise, letting $K = \argmax_{i \in [d]} \frac{g_i}{c_i}$, her manipulation takes the form
\[
y = \bx + \sum_{i=1}^d \frac{t_i}{c_i} \be_i 
\]
for any $\bt$ such that $t_i \ge 0$ for all $i\in [d]$, $t_i = 0$ for all $i \notin K$, and
$\sum_{i=1}^d g_i( x_i + \frac{t_i}{c_i}) = g_0$.
\end{lemma}

\ignore{
\begin{lemma}[$d$-D Candidate Best Response]
\label{lemma:opt_candidate}
Suppose a learner publishes a classifier $f$ based on the hyperplane $\sum_{i=1}^d g_i x_i = g_0$, then a candidate with feature $\bx$ and linear costs $\sum_{i=1}^d c_i x_i$ has a best-response manipulation strategy given by the function $M(\bx)$ where if either $f(\bx) = 1$ or if for all $i \in [d]$, $f(\bx + \frac{1}{c_i} \be_i) = 0$, then $M(\bx) = \bx$. Else, her manipulation takes the form
\begin{align}
\begin{split}
M&(\bx) = \bx + \sum_{i=1} \frac{t_i}{c_i} \be_i \\
&\text { where } t_i \ge 0 \text{ for all } i\in [d], \text{ and } \forall i \notin K, t_i = 0\\
&\text{ and } \sum_{i=1}^d g_i( x_i + \frac{t_i}{c_i}) = g_0
\end{split}
\end{align}
where $K = \argmax_{i \in [d]} \frac{g_i}{c_i}$.
\end{lemma}
} 

While in the $d$-dimensional case, a candidate has many more choices of manipulation directions to pursue, a best response strategy will always lead her to increase her feature in those components that are most valued by the learner and least costly for manipulation. That is, she behaves according to a ``bang for your buck" principle, in which the optimal manipulations are in the direction or directions where the ratio $\frac{g_i}{c_i}$ is highest. 

Despite the fact that the optimal manipulation may not be unique, as in the cases where there are multiple equivalently good directions for a candidate to move in, a learner who knows candidates' costs can still anticipate best-response manipulations and avoid errors on that group. As such, we are once again able to construct a perfect classifier for candidates of group $A$ and a perfect classifier for candidates of group $B$. 

\begin{figure}
\centering 
\includegraphics[width=7cm]{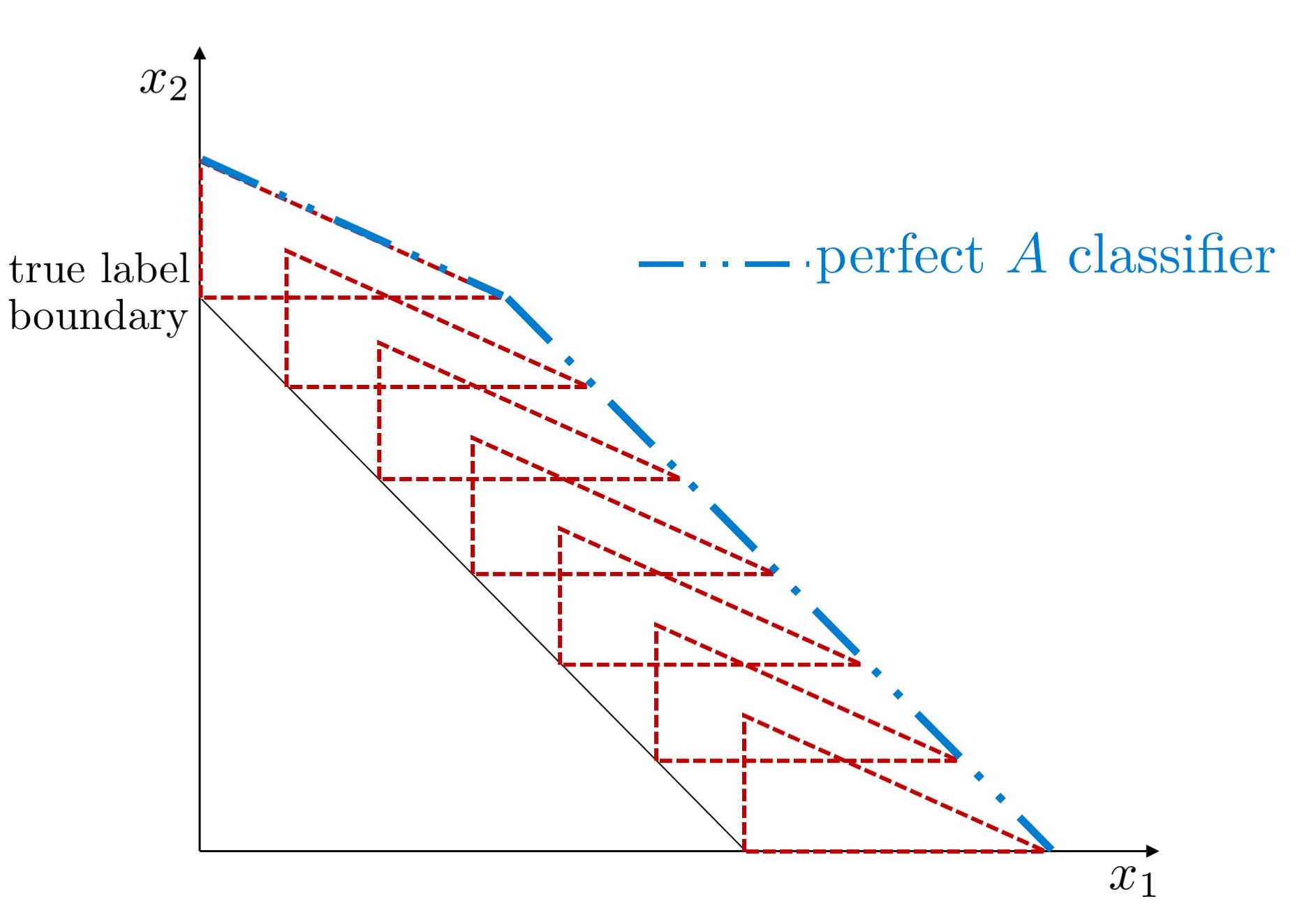}
\caption{A perfect classifier for group $A$. Every candidate with unmanipulated feature vector $\bx$ on or above the true decision boundary for group $A$ is able to manipulate to a point $\by \in \Delta_A(\bx)$ on or above the blue decision boundary depicted here. No candidate with an unmanipulated feature vector below the true decision boundary is able to do so.  The kink in the blue decision boundary arises due to the restriction of features to $[0,1]^d$. A perfect classifier for group $A$ does not need to have this kink; for example, a more lenient perfect classifier can be formed by ``straightening'' it out. }
\label{fig:schmear}
\end{figure}

\begin{theorem}[$d$-D Space of Dominant Learner Strategies]
\label{theorem:d_opt_learner}
In the general $d$-dimensional Strategic Classification Game with linear costs, there exists a classifier that perfectly classifies group $A$ and a classifier that perfectly classifies group $B$. All undominated classifiers commit no false negative errors on group $A$ and no false positive errors on group $B$. 
\end{theorem}

A full exposition of the proof appears in the appendix, but here we present an abbreviated explanation of the result. 

For each group $m$, the learner computes an optimal boundary that perfectly classifies all of its members by considering the set of simplices $\{\Delta_m(\bx)\}$ anchored at the vectors $\bar{\bx}$ that satisfy $\bw_m^\intercal \bar{\bx} = \tau_m$ and drawing the strictest hyperplane that intersects each simplex. That is for all hyperplanes $g_i: \sum_{j=1}^d g_{i,j} x_j = g_{i, 0}$ that are constructed to intersect each simplex, then $g_1: \sum_{j=1}^d g_{1,j} x_j = g_{1, 0}$ is the strictest if for all $\bx \in [0,1]^d$, 
\[\sum_{j=1}^d g_{1,j} x_j = g_{1, 0} \implies \sum_{j=1}^d g_{i,j} x_j = g_{i, 0}\ge g_{j,0}\]
for all $g_i$. Due to the cost ordering, for any $\bx \in [0,1]^d$, $\Delta_B(\bx) \subseteq \Delta_A(\bx)$, and thus wherever a comparison is possible, the group $A$ boundary is at least as strict as the group $B$ boundary. Figure \ref{fig:schmear} gives a visualization of a boundary formed by connecting the simplices $\Delta(\bar{\bx})$; the corresponding classifier perfectly classifies the group.

As in the one-dimensional general costs case, learner strategies necessarily entail inequality-reinforcing classifiers: a rational learner equipped with any error-penalizing cost function will select an equilibrium strategy that trades off undue optimism with respect to group $A$ for undue pessimism with respect to group $B$. We note that except in the extreme case in which there exists a perfect classifier for all candidates in the population, this result implies that the classifier for group $A$ issues false negatives on group $B$, and the classifier for group $B$ issues false positives on group $A$. In order to formalize this result, we would like to generalize the idea behind the minimum correspondence unmanipulated features given by $\ell_A(\cdot)$ and $\ell_B(\cdot)$ in (\ref{correspondence}) for general $d$-dimensions and linear costs. 

A learner who observes a possibly manipulated feature vector $\by$ must consider the space of unmanipulated feature vectors that the candidate could have had. Thus we can make use of the simplex idea of potential manipulation; however in this case, the learner seeks to project a simplex ``backward" to ``undo'' the potential candidate manipulation. Since groups are subject to different costs, simplices $\Delta_A^{-1}(\by)$ and $\Delta_B^{-1}(\by)$---a depiction is given in Figure \ref{b-simplex}---which represent the region from where a candidate could have manipulated, will differ based on the candidate's group membership, with 
\begin{align}
\label{simplex-backward1}
\Delta_A^{-1}(\by)= \Big\{\by - \sum_{i=1}^d \frac{t_i}{c_{A,i}}\be_i \in [0,1]^d  \Big| \sum_{i=1}^d t_i \le 1\text{ ; } t_i \ge 0 \text{ }\forall i \Big\}  ,
\\
\label{simplex-backward2}
\Delta_B^{-1}(\by)= \Big\{ \by - \sum_{i=1}^d\frac{t_i}{c_{B,i}} \be_i \in [0,1]^d  \Big| \sum_{i=1}^d t_i \le 1\text{ ; } t_i \ge 0 \text{ }\forall i \Big\} .
\end{align}
We can now use these constructs in order to define $d$-dimensional generalizations of $\ell_A(\by)$ and $\ell_B(\by)$. 


\begin{definition}[Correspondence with Unmanipulated Features in $d$-D]
For any observed candidate feature $\by \in [0,1]^d$, the minimum corresponding unmanipulated feature vectors are given by 
\begin{align}
\label{correspondence-d}
\ell_A(\by) = \big\{ \bx \in \Delta_A^{-1}(\by) \cap [0,1]^d \big| \nexists  \hat{\bx} \in \Delta_A^{-1}(\by) \text{ such that }  \hat{\bx} <  \bx \big\} ,
\\
\ell_B(\by) = \big\{ \bx \in \Delta_B^{-1}(\by) \cap [0,1]^d \big| \nexists  \hat{\bx} \in \Delta_B^{-1}(\by) \text{ such that } \hat{\bx} <  \bx \big\} 
\end{align}
for a candidate belonging to group $A$ and group $B$ respectively.
\end{definition}
The corresponding values $\ell_A(\by)$ and $\ell_B(\by)$ are defined such that a candidate who presents feature $\by$ must have had a true unmanipulated feature vector $\bx \ge \bar{\bx}$ for some $\bar{\bx} \in \ell_A(\by)$ if she is a group $A$ member and $\bx \ge \bar{\bx}$ for some $\bar{\bx} \in \ell_B(\by)$ if she is a group $B$ member. 
\begin{figure}
\centering
\includegraphics[width=4.5cm]{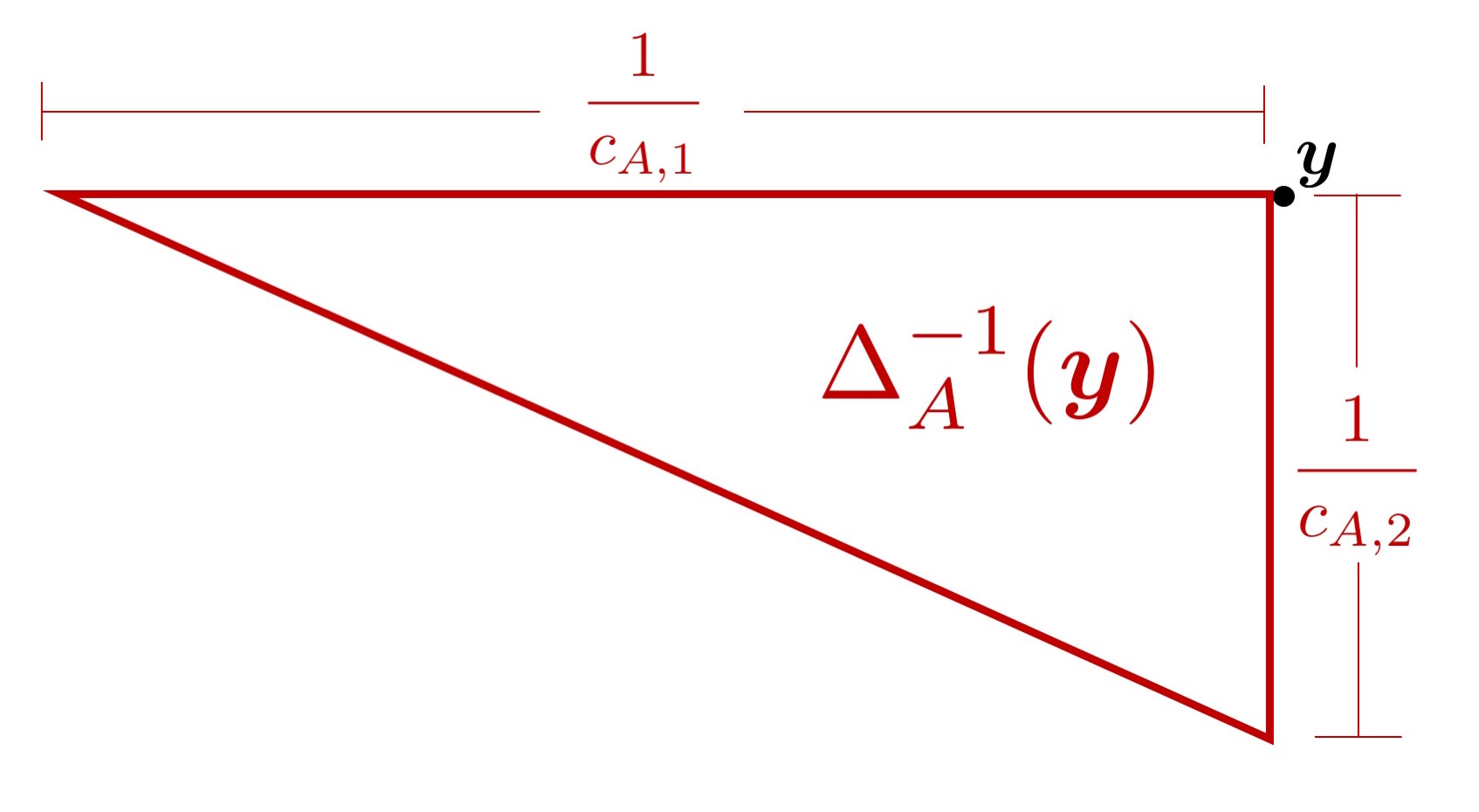}
\caption{The backward simplex. A candidate in group $A$ with manipulated feature vector $\by$ could have started with any feature vector $\bx \in \Delta_A^{-1}(\by)$ and paid a cost of at most 1.}
\label{b-simplex}
\end{figure}

For any hyperplane decision boundary $g$ containing vectors $\by$, the minimum corresponding feature vectors given by $\ell_A(\by)$ and $\ell_B(\by)$ are helpful for determining the effective thresholds that $g$ generates on unmanipulated features for groups $A$ and $B$. 

\begin{lemma}
\label{linear-d}
Suppose a learner classifier $f$ is based on a hyperplane $g: \sum_{i=1}^d g_i x_i = g_0$. Construct the set 
\begin{equation}
\mathcal{L}_m(g) = \left\{ \argmin_{\bx \in \ell_m(\by)} \sum_{i=1}^d g_i x_i \Big|  \forall \by \text{ s. t. } \sum_{i=1}^d g_i y_i = g_0 \right\}
\end{equation}
Then a group $m$ agent with feature $\bx$ can move to some $\by$ with $f(\by) = 1$ and $c_m(\by) - c_m(\bx) \leq 1$ if and only if $\bx \ge \ell $ for some $\ell \in \mathcal{L}_m(g)$.
\end{lemma}


By definition, for any two $\ell_1, \ell_2 \in \mathcal{L}_m(g)$, \[\sum_{i=1}g_i \ell_{1,i} =\sum_{i=1}g_i \ell_{2,i} = g_0 -\frac{g_{k_m}}{c_{m,k_m}},\] where $ k_m \in \argmax_{i=[d]} \frac{g_{i}}{c_{m,i}}$. Thus a learner who cares only about the true label of presented features, will construct her decision boundary $g$ such that all $\ell \in \mathcal{L}_m(g)$ have the same true label. 

A cost-minimizing learner who publishes a classifier $f$ based on a hyperplane $g$ on manipulated features will commit errors on those candidates with unmanipulated features $\bx \in [0,1]^d$ contained within the boundaries given by $\mathcal{L}_A(g)$ and $\mathcal{L}_B(g)$. This space can be understood as the $d$-dimensional generalization of the $[\ell_A(\sigma), \ell_B(\sigma)]$ error interval in one-dimension. 

\begin{proposition}[Learner's Cost in $d$ Dimensions]
\label{cost-d-dim}
A learner who publishes an undominated classifier $f$ based on a hyperplane $\bm{g}^\intercal \bx = g_0$ can only commit false positives on group $A$ candidates and false negatives on group $B$ candidates. The cost of such a classifier is
\begin{equation*}
\begin{split}
&C_{FN} P_{x \sim \mathcal{D}_B}\Big[ \bx \in \big(\bm{g}^\intercal \bx < g_0 - \frac{g_{k_B}}{c_{k_B}} \bigcap \bw_B^\intercal \bx \ge \tau_{B}\ \big) \Big] \\
&+ C_{FP} P_{x \sim \mathcal{D}_A}\Big[\bx \in \big(\bw_A^\intercal \bx < \tau_{A} \bigcap\bm{g}^\intercal \bx \ge g_0  - \frac{g_{k_A}}{c_{k_A}} \big) \Big] ,
\end{split}
\end{equation*}
where $k_B \in \argmax_{i\in[d]} \frac{g_i}{c_{B,i}}$ and $k_A \in \argmax_{i\in[d]}\frac{g_i}{c_{A,i}}$.
\end{proposition}

\section{Learner Subsidy Strategies}
\label{sec:subs}
Since in our setting, the learner's classification errors are directly tied to unequal group costs, we ask whether she would be willing to subsidize group $B$ candidates in order to shrink the manipulation gap between the two groups and as a result, reduce the number of errors she commits. In this section, we formalize subsidies as interventions that a learner can undertake to improve her classification performance. Although in many high-stakes classification settings, the barriers that make manipulation differentially accessible are non-monetary---such as time, information, and social access---in this section, we consider subsidies that are monetary in nature to alleviate the financial burdens of manipulation.

We introduce these subsidies for the purpose of analyzing their effects on not only the learner's classification performance but also candidate groups' outcomes. Since subsidies mitigate the inherent disparities in groups' costs and increase access to manipulation, one might expect that their implementation would surely improve group $B$'s overall welfare. In this section, we show that in some cases, optimal subsidy interventions can surprisingly have the effect of lowering the welfare of candidates from \emph{both} groups without improving the welfare of even a single candidate.

\subsection{Subsidy Formalization}
There are different ways in which a learner might choose to subsidize candidates costs. In the main text of this paper, we focus on subsidies that reduce each group $B$ candidate's costs such that the agent need only pay a $\beta$ fraction of her original manipulation cost. 

\begin{definition}[Proportional subsidy]
Under a proportional subsidy plan, the learner pays a proportion $1-\beta$ of each group B candidate's cost of manipulation for some $\beta \in [0, 1]$. As such, a group B candidate who manipulates from an initial feature vector $\bx$ to a final feature vector $\by$ bears a cost of $\beta\big(c_B(\by) - c_B(\bx)\big)$.
\end{definition}

In the appendix, we also introduce flat subsidies in which the learner absorbs up to a flat $\alpha$ amount from each group $B$ candidate's costs, leaving the candidate to pay  $\max\{0, c_B(\by) - c_B(\bx) - \alpha\}$. Similar results to those shown in this section hold for flat subsidies. 

When considering proportional subsidies, the learner's strategy now consists of both a choice of $\beta$ and a choice of classifier $f$ to issue. The learner's goal is to minimize her penalty
\begin{equation*}
\begin{split}
&C_{FP} \sum_{m \in \{A, B\}} p_m P_{\bx\sim\mathcal{D}_m}\big[h_m(\bx) = 0, f(\by) = 1\big] \\
& + C_{FN}\sum_{m \in \{A, B\}} p_m P_{\bx\sim\mathcal{D}_m}\big[h_m(\bx) = 1, f(\by) = 0\big] 
 + \lambda cost(f, \beta),
\end{split}
\end{equation*}
where $ cost(f, \beta)$ is the monetary cost of the subsidy, $C_{FP}$ and $C_{FN}$ denote the cost of a false positive and a false negative respectively as before, and $\lambda \ge 0 $ is some constant that determines the relative weight of misclassification errors and subsidy costs for the learner.

For ease of exposition, the remainder of the section is presented in terms of one-dimensional features. In Section \ref{reduction} of the appendix, we show that in many cases, the $d$-dimensional linear costs setting can be reduced to this one-dimensional setting. 

As an analog of (\ref{correspondence}), we define $\ell_B^\beta(y) = (\beta c_B)^{-1}(\beta c_B(y) -1)$, giving the minimum corresponding unmanipulated feature $x$ for any observed feature $y$. Under the proportional subsidy, for a given $y$, the group $B$ candidate must have $x \ge \ell^\beta_B(y)$. From this, we define $\sigma_B^\beta$ such that $\ell_{B}^\beta(\sigma_B^\beta) = \tau_B$. 

In order to compute the cost of a subsidy plan, we must determine the number of group $B$ candidates who will take advantage of a given subsidy benefit. Since manipulation brings no benefit in itself, candidates will only choose to manipulate and use the subsidy if it will lead to a positive classification. For a published classifier $f$ with threshold $\sigma$, we then have 
\[
cost(f, \beta)  = \big(1-\beta\big)\int_{\ell_B^\beta(\sigma)}^\sigma \big(c_B(\sigma) - c_B(x)\big) P_{x\sim \mathcal{D}_B}(x) dx.
\]

Although the learner's optimization problem can be solved analytically for various values of $\lambda$, we are primarily interested in taking a welfare-based perspective on the effects of various classification regimes on both the learner and candidate groups. In the following section, we analyze how the implementation of a subsidy plan can alter a learner's classification strategy and consider the potential impacts of such policies on candidate groups.

\subsection{Group Welfare Under Subsidy Plans}

While a learner would choose to adopt a subsidy strategy primarily in order to reduce her error rate, offering cost subsidies can also be seen as an intervention that might equalize opportunities in an environment that by default favors those who face lower costs. That is, if costs are keeping group $B$ down, then one might believe that reducing  costs will surely allow group $B$ a fairer shot at manipulation, and, as a result, a fairer shot at positive classification. Alas we find that mitigating cost disparities by way of subsidies does not necessarily lead to better outcomes for group $B$ candidates. In fact, an optimal subsidy plan can actually reduce the welfares of \textit{both} groups. Paradoxically, in some cases, the subsidy plan boosts only the learner's utility, whereas every individual candidate from both groups would have preferred that she offer no subsidies at all. 

The following theorem captures the surprising result that subsidies can be harmful to all candidates, even those from the group that would appear to benefit.

\begin{theorem}[Subsidies can harm both groups]
\label{surprising1}
There exist cost functions $c_A$ and $c_B$ satisfying the cost conditions, learner distributions $\mathcal{D}_A$ and $\mathcal{D}_B$, true classifiers with threshold $\tau_A$ and $\tau_B$, population proportions $p_A$ and $p_B$, and learner penalty parameters $C_{FN}$, $C_{FP}$, and $\lambda$, such that
no candidate in either group has higher payoff at the equilibrium of the Strategic Classification Game with proportional subsidies compared with the equilibrium of the Strategic Classification Game with no subsidies, and some candidates from both group $A$ and group $B$ are strictly worse off.
\end{theorem}

We note that a slightly weaker version of the theorem holds for flat subsidies.  In particular, there exist cases in which some individual candidates have higher payoff at the equilibrium of the Strategic Classification Game with flat subsidies compared with the equilibrium with no subsidies, but both group $A$ and group $B$ candidates have lower payoffs on average with the subsidies.

To prove the theorem, it suffices to give a single case in which both candidate groups are harmed by the use of subsidies. However, to illustrate that this phenomenon does not arise only as a rare corner case, we provide one such example here plus two in the appendix, and discuss general conditions under which this occurs. In each example, we consider a particular instance of the Strategic Classification Game and compare the welfares of candidates at equilibrium when the learner is able to select a proportional subsidy with their welfares at equilibrium when no subsidy is allowed. 


\begin{example} 
\label{example-concave} 
Suppose that a learner is error-minimizing such that $C_{FN} = C_{FP} = 1$ and $\lambda = \frac{3}{4}$. Suppose that unmanipulated features for both groups are uniformly distributed with $p_A = p_B = \frac{1}{2}$. Let group cost functions be given by $c_A(x) = 8\sqrt{x} + x$ and $c_B(x) = 12\sqrt{x}$; note that the cost condition $c'_A(x) < c'_B(x)$ holds for $x \in [0,1]$. Let the true group thresholds be given by $\tau_A = 0.4$ and $\tau_B = 0.3$.


When subsidies are not allowed, the learner chooses a classifier with threshold $\sigma^* = \sigma_B \approx 0.398$ at equilibrium.
This threshold perfectly classifies all candidates from group $B$, while permitting false positives on candidates from group $A$ with features $x \in [0.272, 0.4)$.

If the learner decides to implement a proportional subsidies plan, at equilibrium the learner chooses a classifier with threshold $\sigma^*_{prop} = \sigma_A \approx 0.546$ and a subsidy parameter $\beta^* =  0.558$.
Her new threshold now correctly classifies all members of group $A$, while committing false negatives on group $B$ members with features 
$x \in [0.3, 0.348)$.

Some candidates in group $B$ are thus strictly worse-off, while none improve. Without the subsidy offering, group B members had been perfectly classified, but now there exist some candidates who are mistakenly excluded. Further, one can show that candidates who are positively classified must pay more to manipulate to the new threshold in spite of receiving the subsidy benefit. This increased cost is due to the fact that the higher classification threshold imposes greater burdens on manipulation than the $\beta$ subsidy alleviates.

Group $A$ candidates are also strictly worse-off since the threshold increase eliminates false positive benefits that some members had previously been granted in the no-subsidy regime. Moreover, all candidates who manipulate must expend more to do so, since these candidates do not receive a subsidy payment.
 Only the learner is strictly better off with the implementation of this subsidy plan. 
\end{example}

Additional examples in the appendix show cases in which both groups experience diminished welfare when they bear linear costs. Even when the learner has an error function that penalizes false negatives twice as harshly as false positives and thus is explicitly concerned with mistakenly excluding group B candidates, an equilibrium subsidy strategy can still make both groups worse-off.

We thus highlight two consequences of subsidy interventions: On the one hand, with reduced cost burdens, more candidates from the disadvantaged group should be able to manipulate to reach a positive classification. However, subsidy payments also allow a learner to select a classifier that is at least as strict as the one issued without offering subsidies. These are opposing forces, and these examples show that without needing to  distort underlying group probability distributions or the learner's penalty function in extreme ways, the effect of mitigating manipulation costs may be outweighed by the overall impact of a stricter classifier.

This result can also be extended to show that a setup in which candidates are unable to manipulate their features at all can be preferred by all three parties---groups $A$ and $B$ as well as the learner---to both the manipulation and subsidy regimes. We provide an informal statement of this proposition below and defer the interested reader to its formal statement and demonstration in the appendix.

\begin{proposition}
\label{surprising2}
There exist general cost functions such that the outcomes issued by a learner's equilibrium classifier under a non-manipulation regime is preferred by all parties---the learner, group $A$, and group $B$---to outcomes that arise both under her equilibrium manipulation classifier and under her equilibrium subsidy strategy.
\end{proposition}

\section{Discussion}

Social stratification is constituted by forms of privilege that exist along many different axes, weaving and overlapping to create an elaborate mesh of power relations. While our model of strategic manipulation does not attempt to capture this irreducible complexity, we believe this work highlights a likely consequence of the expansion of algorithmic decision-making in a world that is marked by deep social inequalities. We demonstrate that the design of classification systems can grant undue rewards to those who \textit{appear} more meritorious under a particular conception of merit while justifying exclusions of those who have failed to meet those standards. These consequences serve to exacerbate existing inequalities. 

Our work also shows that attempts to resolve these negative social repercussions of classification, such as implementing policies that help disadvantaged populations manipulate their features more easily, may actually have the opposite effect. A learner who has offered to mitigate the costs facing these candidates may be encouraged to set a higher classification standard, underestimating the deeper disadvantages that a group encounters, and thus serving to further exclude these populations. However, it is important to note that these unintended consequences do not always arise. A conscientious learner who offers subsidies to equalize the playing field can guard against such paradoxes by making sure to classify agents in the same way even when offering to mitigate costs.

Other research in signaling and strategic classification has considered models in which manipulation is desirable from the learner's point of view \cite{frankel2014muddled,kleinberg2018classifiers}. Though this perspective diverges from the one we consider here, we acknowledge that there do exist cases in which manipulation serves to improve a candidate's quality and thus leads a learner to encourage such behaviors. It is important to note, however, that although this account may accurately represent some social classification scenarios, differential group access to manipulation remains an issue, and in fact, cases in which manipulation genuinely improves candidate quality may present even more problematic scenarios for machine learning systems. As work in algorithmic fairness has shown, feedback effects of classification can lead to deepening inequalities that become ``justified" on the basis of features both manipulated and ``natural" \cite{ensign2018runaway}. 

The rapid adoption of algorithmic tools in social spheres calls for a range of perspectives and approaches that can address a variety of domain-specific concerns. Expertise from other disciplines ought to be imported into machine learning, informing and infusing our research in motivation, application, and technical content. As such, our work seeks to investigate, from a theoretical learning perspective, some of the potential adverse effects of what sociology has called ``quantification," a world increasingly governed by metrics. In doing so, we bring in techniques from game theory and information economics to model the interaction between a classifier and its subjects. This paper adopts a framework that tries to capture the genuine unfair aspects of our social reality by modeling group inequality in a population of agents. Although this perspective deviates from standard idealized settings of learner-agent interaction, we believe that so long as machine learning tools are designed for deployment in the imperfect social world, pursuing algorithmic fairness will require us to explicitly build models and theory to address critical issues such as social stratification and unequal access.

\section*{Acknowledgements}
We thank Alex Frankel, Rupert Freeman, Manish Raghavan, Hanna Wallach, and Glen Weyl for constructive input and discussion on this project and related topics.

\bibliographystyle{unsrt}
\bibliography{sample}

\begin{thebibliography}{10}

\bibitem{sweeney2013discrimination}
Latanya Sweeney.
\newblock Discrimination in online ad delivery.
\newblock {\em Queue}, 11(3):10, 2013.

\bibitem{o2016weapons}
Cathy O'Neil.
\newblock {\em Weapons of Math Destruction: How Big Data Increases Inequality
  and Threatens Democracy}.
\newblock Broadway Books, 2016.

\bibitem{angwin2016machine}
Julia Angwin, Jeff Larson, Surya Mattu, and Lauren Kirchner.
\newblock Machine bias.
\newblock {\em ProPublica, May}, 23, 2016.

\bibitem{eubanks2018automating}
Virginia Eubanks.
\newblock {\em Automating inequality: How High-tech Tools Profile, Police, and
  Punish the Poor}.
\newblock St. Martin's Press, 2018.

\bibitem{johnson2016impartial}
Kory~D Johnson, Dean~P Foster, and Robert~A Stine.
\newblock Impartial predictive modeling: {E}nsuring fairness in arbitrary
  models.
\newblock CoRR arXiv:1608.00528, 2016.

\bibitem{qureshi2016causal}
Bilal Qureshi, Faisal Kamiran, Asim Karim, and Salvatore Ruggieri.
\newblock Causal discrimination discovery through propensity score analysis.
\newblock CoRR arXiv:1608.03735, 2016.

\bibitem{kilbertus2017avoiding}
Niki Kilbertus, Mateo~Rojas Carulla, Giambattista Parascandolo, Moritz Hardt,
  Dominik Janzing, and Bernhard Sch{\"o}lkopf.
\newblock Avoiding discrimination through causal reasoning.
\newblock In {\em Advances in Neural Information Processing Systems}, 2017.

\bibitem{grgic2018beyond}
Nina Grgic-Hlaca, Muhammad~Bilal Zafar, Krishna~P Gummadi, and Adrian Weller.
\newblock Beyond distributive fairness in algorithmic decision making:
  {F}eature selection for procedurally fair learning.
\newblock In {\em Proceedings of the AAAI Conference on Artificial
  Intelligence}, 2018.

\bibitem{bruckner2011stackelberg}
Michael Br{\"u}ckner and Tobias Scheffer.
\newblock Stackelberg games for adversarial prediction problems.
\newblock In {\em Proceedings of the ACM SIGKDD International Conference on
  Knowledge Discovery and Data Mining}, 2011.

\bibitem{hardt2016strategic}
Moritz Hardt, Nimrod Megiddo, Christos Papadimitriou, and Mary Wootters.
\newblock Strategic classification.
\newblock In {\em Proceedings of the ACM Conference on Innovations in
  Theoretical Computer Science}, 2016.

\bibitem{akyol2016price}
Emrah Akyol, Cedric Langbort, and Tamer Basar.
\newblock Price of transparency in strategic machine learning.
\newblock CoRR arXiv:1610.08210, 2016.

\bibitem{datta2017proxy}
Anupam Datta, Matt Fredrikson, Gihyuk Ko, Piotr Mardziel, and Shayak Sen.
\newblock Proxy non-discrimination in data-driven systems.
\newblock CoRR arXiv:1707.08120, 2017.

\bibitem{kearns1993learning}
Michael Kearns and Ming Li.
\newblock Learning in the presence of malicious errors.
\newblock {\em SIAM Journal on Computing}, 22(4):807--837, 1993.

\bibitem{auer1998line}
Peter Auer and Nicolo Cesa-Bianchi.
\newblock On-line learning with malicious noise and the closure algorithm.
\newblock {\em Annals of mathematics and artificial intelligence},
  23(1-2):83--99, 1998.

\bibitem{dong2018strategic}
Jinshuo Dong, Aaron Roth, Zachary Schutzman, Bo~Waggoner, and Zhiwei~Steven Wu.
\newblock Strategic classification from revealed preferences.
\newblock In {\em Proceedings of the ACM Conference on Economics and
  Computation}, 2018.

\bibitem{spence1978job}
Michael Spence.
\newblock Job market signaling.
\newblock In {\em Uncertainty in Economics}, pages 281--306. 1978.

\bibitem{frankel2014muddled}
Alex Frankel and Navin Kartik.
\newblock Muddled information.
\newblock {\em Journal of Political Economy}, Forthcoming, 2018.

\bibitem{esteban2006inequality}
Joan Esteban and Debraj Ray.
\newblock Inequality, lobbying, and resource allocation.
\newblock {\em American Economic Review}, 96(1):257--279, 2006.

\bibitem{kephart2015complexity}
Andrew Kephart and Vincent Conitzer.
\newblock Complexity of mechanism design with signaling costs.
\newblock In {\em Proceedings of the International Conference on Autonomous
  Agents and Multiagent Systems}, 2015.

\bibitem{kephart2016revelation}
Andrew Kephart and Vincent Conitzer.
\newblock The revelation principle for mechanism design with reporting costs.
\newblock In {\em Proceedings of the ACM Conference on Economics and
  Computation}, 2016.

\bibitem{milli2018social}
Smitha Milli, John Miller, Anca~D Dragan, and Moritz Hardt.
\newblock The social cost of strategic classification.
\newblock Forthcoming, 2019.

\bibitem{card2007racial}
David Card and Jesse Rothstein.
\newblock Racial segregation and the black--white test score gap.
\newblock {\em Journal of Public Economics}, 91(11--12):2158--2184, 2007.

\bibitem{spremann1987agent}
Klaus Spremann.
\newblock Agent and principal.
\newblock In {\em Agency theory, information, and incentives}, pages 3--37.
  Springer, 1987.

\bibitem{laffont2009theory}
Jean-Jacques Laffont and David Martimort.
\newblock {\em The Theory of Incentives: The Principal-Agent Model}.
\newblock Princeton University Press, 2009.

\bibitem{ballwieser2012agency}
Wolfgang Ballwieser, G~Bamberg, MJ~Beckmann, H~Bester, M~Blickle, R~Ewert,
  G~Feichtinger, V~Firchau, F~Fricke, H~Funke, et~al.
\newblock {\em Agency theory, information, and incentives}.
\newblock Springer Science \& Business Media, 2012.

\bibitem{kleinberg2018classifiers}
Jon Kleinberg and Manish Raghavan.
\newblock How do classifiers induce agents to invest effort strategically?
\newblock CoRR arXiv:1807.05307, 2018.

\bibitem{ensign2018runaway}
Danielle Ensign, Sorelle~A Friedler, Scott Neville, Carlos Scheidegger, and
  Suresh Venkatasubramanian.
\newblock Runaway feedback loops in predictive policing.
\newblock In {\em Proceedings of the Conference on Fairness, Accountability and
  Transparency}, 2018.

\end{thebibliography}

\appendix

\section{Appendix}

\subsection{Proofs from Section~\ref{sec:eq1d}}

\subsubsection{Proof of Proposition \ref{1d-optimal}}

We first construct the optimal learner classifier when facing only candidates of a single group. Suppose the learner encounters only group $A$ candidates. Then using her knowledge that the true classifier $h_A$ is based on a threshold $\tau_A \in [0,1]$, she can construct a classifier that admits those candidates with scores $x\ge \tau_A$ and rejects candidates $x <  \tau_A$. Since the maximal manipulation cost that any candidate would be willing to undertake is $1$, for all $x \in [0,1]$, 
$c_A(y) - c_A(x) \le 1$ and therefore
\[y \le c_A^{-1}(c_A(x)+1)\]
Thus a candidate with feature $x = \tau_A$ would be able to move to any feature $y \le \sigma_A$ where $\sigma_A = c_A^{-1}(c_A(\tau_A)+1)$. 

Repeating the same reasoning for group $B$, a candidate with feature $x = \tau_B$ would be willing to move to any feature $y \le \sigma_B$ where $\sigma_B = c_B^{-1}(c_B(\tau_B) + 1)$. 

Now we want to show that $[\sigma_B, \sigma_A]$ marks an interval of undominated strategies. First we prove the ordering that $\sigma_B \le \sigma_A$ for all cost functions $c_B$ and $c_A$ and all thresholds $\tau_B \le \tau_A$. Recall that since $h_A(x) = 1 \implies h_B(x) = 1$, we have $\tau_B \le \tau_A$. Although we cannot order $c_B(\tau_B)$ and $c_A(\tau_A)$, we have, by monotonicity of $c_B$
\[ c_B(\tau_B) \le c_B(\tau_A).\] 
Let $\Delta = c_B(\tau_A) - c_B(\tau_B)$. Notice that if $\Delta \ge 1$, 
$c_B(\tau_B) + 1 \le c_B(\tau_A)$, and so
\[\sigma_B = c_B^{-1}(c_B(\tau_B) + 1 ) \le \tau_A < \sigma_A , \]
where the last inequality is due to monotonicity of $c_A$. \checkmark

Let us consider the $\Delta \in (0,1)$ case. By the cost condition, we can write $c_B' (\tau_A) \ge c_A'(\tau_A)$. This implies that
\[c_B^{-1}(c_B(\tau_A) + 1 ) \le c_A^{-1}(c_A(\tau_A) + 1 )\]
Substituting in $c_B(\tau_A) = c_B(\tau_B) + \Delta$, we have
\[c_B^{-1}( c_B(\tau_B) + \Delta + 1 ) \le c_A^{-1}(c_A(\tau_A) + 1 ) = \sigma_A.\]
By monotonicity of $c_B$, the left hand side is $\ge \sigma_B$, and we have that
$\sigma_B \le \sigma_A$
as desired. \checkmark

Notice that for all $\sigma < \sigma_B$, the learner commits false positive errors on candidates from group $B$, since $\sigma_B$ is optimal for group $B$ classification. She commits more false positives on group $A$ candidates as well and does not commit any fewer false negatives because of the monotonicity of $c_B$ and $c_A$. Thus for any error function with $C_{FP} > 0$, the threshold classifier $\sigma_B$ dominates $\sigma$. 

Similarly, for all $\sigma > \sigma_A$, the learner commits false negative errors on candidates from group $A$, since $\sigma_A$ is optimal for group $A$ classification. She also commits more false negatives on group $B$ while committing no fewer false positives. Thus for any error function with $C_{FN} > 0$, the threshold classifier $\sigma_A$ dominates $\sigma$. 

For all $\sigma \in [\sigma_B, \sigma_A]$, the learner trades off false negatives on group $B$ for false positives on group $A$, and we call this range of threshold strategies undominated. 
\qed

\subsubsection{Proof of Proposition \ref{1d-cost}}

We compute the cost of a learner's threshold strategy $\sigma \in [\sigma_B, \sigma_A]$ by first examining its performance on each group individually. 

Recall from Proposition \ref{1d-optimal} that the optimal learner threshold that perfectly classifies all $B$ candidates is $\sigma_B$. Thus for all threshold strategies based on $\sigma \in (\sigma_B, \sigma_A]$, the learner commits false negative errors on group $B$. 

To compute which members of group $B$ are subject to these errors, consider a learner classifier $f$ based on a threshold $\sigma$. In order to manipulate to reach the feature threshold $\sigma$, a group $B$ candidate must have an unmanipulated $x$ such that
\[ c_B(\sigma) - c_B(x) \le 1, \]
\[ x \ge c_B^{-1}(c_B(\sigma) + 1 ) = \ell_B(\sigma).\]
We know that $\tau_B \le \ell_B(\sigma)$ by monotonicity of $c_B$, and thus for all group $B$ candidates with feature $x \in [ \tau_B, \ell_B(\sigma) )$, the learner issues classification $f(x) = 0$, even though $h_B(x) = 1$. These are the false negative errors issued on group $B$ for which the learner bears cost
\begin{equation}
\label{c-fn}
C_{FN} p_B P_{x \sim \mathcal{D}_B}\big[x \in [ \tau_B, \ell_B(\sigma)) \big]
\end{equation}

Following the same reasoning, notice that since $\sigma_A$ is the optimal threshold policy for a learner facing only group $A$ candidates, a classifier $f$ based on any $\sigma \in [\sigma_B, \sigma_A)$ commits false positive errors on some group $A$ candidates. Then repeating the steps that we carried out for group $B$, we see that for all group $A$ candidates with $x$ such that
\[x \ge c_A^{-1}(c_A(\sigma)+1 = \ell_A(\sigma)\]
the classifier $f$ issues a positive classification; $f(x) = 1$. Since $\ell_A(\sigma) \leq \tau_A$, candidates with features $x \in [\ell_A(\sigma), \tau_A)$, have true label $h_A(x) = 0$, and the learner commits false positive errors that bear cost
\begin{equation}
\label{c-fp}
C_{FP} p_A P_{x \sim \mathcal{D}_A}\big[x \in [\ell_A(\sigma), \tau_A) \big]
\end{equation}
Combining (\ref{c-fn}) and (\ref{c-fp}), the total cost of any classifier $f$ based on a threshold $\sigma \in [\sigma_B, \sigma_A]$, we obtain our desired result. 
\qed

\subsubsection{Proofs of Corollaries \ref{corollary-fn} and \ref{corollary-fp}}
These results follow by considering strategies $\sigma_B$, which commits no errors on group $B$ and thus only bears the cost given in (\ref{c-fp}), and $\sigma_A$, which commits no errors on group $A$ and thus only bears the cost given in (\ref{c-fn}).
\qed

\subsubsection{Proof of Proposition \ref{proportional-costs}}

Under the assumption of uniform feature distributions for both groups, minimizing a classifier's probability of error amounts to choosing the threshold $\sigma$ as
\[\argmin_{\sigma\in [\sigma_B, \sigma_A]} \ell_B(\sigma) - \ell_A(\sigma).\]

 
With proportional group costs $c_A(x) = q c_B(x)$ for $q\in (0,1)$, we have that 
\begin{align*}
\ell_B'(\sigma) &= \frac{(c_B)'(\sigma)}{\Big(c_B\Big)'\Big(\big(c_B\big)^{-1}\big(c_B(\sigma)-1\big)\Big)}\\
&=\frac{(c_B)'(\sigma)}{\Big(c_B\Big)'\Big(\ell_B(\sigma)\Big)}
\end{align*}
and
\begin{align*}
\ell_A'(\sigma) &= \frac{(c_A)'(\sigma)}{\Big(c_A\Big)'\Big(\big(c_A\big)^{-1}\big(c_A (\sigma)-1\big)\Big)} \\
&= \frac{(q c_B)'(\sigma)}{\Big(q c_B\Big)'\Big(\big(c_A\big)^{-1}\big(c_A (\sigma)-1\big)\Big)} \\
&= \frac{(c_B)'(\sigma)}{\Big(c_B\Big)'\Big(\big(c_A\big)^{-1}\big(c_A (\sigma)-1\big)\Big)} \\
&= \frac{(c_B)'(\sigma)}{\Big(c_B\Big)'\Big(\ell_A(\sigma)\Big)} . 
\end{align*}

When $c_A$ and $c_B$ are strictly concave, since $\ell_B(\sigma) > \ell_A(\sigma)$, $(c_B)'(\ell_A(\sigma)) > (c_B)'(\ell_B(\sigma))$ and therefore
$\ell_A'(\sigma) < \ell_B'(\sigma)$ for all $\sigma \in [\sigma_B, \sigma_A]$, and the quantity $\ell_B(\sigma) - \ell_A(\sigma)$ is monotonically increasing in $\sigma$. Thus the optimal classifier threshold is $\sigma^* = \sigma_B$.

Similarly, when $c_A$ and $c_B$ are strictly convex, $\ell_A'(\sigma) > \ell_B'(\sigma)$ for all $\sigma \in [\sigma_B, \sigma_A]$, and the quantity $\ell_B(\sigma) - \ell_A(\sigma)$ is monotonically decreasing in $\sigma$. Thus the optimal classifier threshold is $\sigma^* = \sigma_A$. Thus the optimal classifier threshold is $\sigma^* = \sigma_A$.

Finally, when $c_A$ and $c_B$ are affine, $\ell_A'(\sigma) = \ell_B'(\sigma)$ for all $\sigma \in [\sigma_B, \sigma_A]$, and the quantity $\ell_B(\sigma) - \ell_A(\sigma)$ is constant for all $\sigma \in[\sigma_B, \sigma_A]$. Thus the learner is indifferent between all thresholds $\sigma \in [\sigma_B, \sigma_A]$.
\qed

\subsection{Proofs from Section~\ref{sec:eqdd}}

\subsubsection{Proof of Lemma \ref{lemma:opt_candidate}}

Consider a candidate with unmanipulated feature $\bx \in [0,1]^d$ and manipulation cost $\sum_{i=1}^d c_i x_i$ who faces a classifier $f(\by)$ with linear decision boundary given by $\sum_{i=1}^d g_i y_i = g_0$. Recall that the utility a candidate receives for presenting feature $\by \ge \bx$ is given by $f(\by) - c(\bx, \by)$. When $f(\bx) = 1$, it is trivial that the candidate's best response to select $\by=\bx$. \checkmark\\

Notice that if for all $i \in [d]$, $f(\bx+ \frac{1}{c_i}\be_i) = 0$, then we have that $\bm{g}^\intercal \bx + \frac{g_k}{c_k} < g_0$, so 
\[\frac{c_k(g_0 - \bm{g}^\intercal \bx)}{g_k} > 1\]
The manipulation from $\bx$ to $\by = \bx + \sum_{i\in K} \frac{t_i}{c_i}\be_i$ such that $\bm{g}^\intercal\by = g_0$ entails cost 
\[c(\by) - c(\bx) = \sum_{i \in K} t_i = \frac{c_k(g_0 - \bm{g}^\intercal \bx)}{g_k} > 1\]
and manipulating to achieve a positive classification using only components in $K$ would require a cost $>$ 1. By definition, keeping the sum $\sum_{i\in K} t_i$, but selecting different $t_i$ such that some $i\notin K$, $t_i > 0 $ would yield an even lower value $\bm{g}^\intercal \bx + \sum_{i=1}^d \frac{g_i t_i}{c_i}$.

Thus manipulating from $\bx$ to $\by$ such that $f(\by) = 1$ entails a cost $c(\by) - c(\bx) > 1$, and the candidate would not move at all, since the utility for moving $1 - (c(\by) - c(\bx)) < 0$ makes her worse-off than being subject to a negative classification without expending any cost on feature manipulation. Thus she selects $\by = \bx$. \checkmark \\

Now we consider the case where $f(\bx) = 0$ and there exists $i \in [d]$ such that $f(\bx+ \frac{1}{c_i}\be_i) = 1$.

Let $k \in K = \argmax_{i\in [d]} \frac{g_i}{c_i}$. We prove that the best-response manipulation for candidates with these $\bx$ moves to 
\begin{equation}
\label{Mx-def}
\by = \bx+ \sum_{i =1}^d \frac{t_i}{c_i}\be_i
\end{equation}
 where $t_i \ge 0$, $t_j = 0$ for all $j\notin K$, and $\bm{g}^\intercal (\bx + \sum_{i \in K}\frac{t_i}{c_i}\be_i) = g_0$. Note that such a $\by$ may not be unique---there may be multiple best-response manipulated features that achieve the same candidate utility, since they all result in the same candidate cost, and thus regardless of choices $i \in K$, we have that
 \begin{equation}
 \label{sum-t}
 \sum_{i \in K} t_i = \frac{c_k(g_0 - \bm{g}^\intercal \bx)}{g_k}
 \end{equation}

The utility of any move to $\by$ satisfying (\ref{Mx-def}) is given by \[f(\by^*) - c(\bx, {\by}^*) = 1- \sum_{i=1}t_i\] 
Let us pick any such $\by$ and call it $\by^*$ since we will show that all other manipulations that are not of the form given in (\ref{Mx-def}) generate lower utility for the candidate than $\by^*$.  

We now show that for any manipulation to $\by$, $\sum_{i=1}^d t_i \le 1$. By assumption, for some $i$, we have \[f(\bx + \frac{1}{c_i}\be_i) = 1 \implies \bm{g}^\intercal \bx + \frac{g_i}{c_i} \ge g_0\] 
Thus by (\ref{sum-t}), we have that $ \sum_{i \in K} t_i \le \frac{c_k \frac{g_i}{c_i}}{g_k}$.
By definition of $k,$ 
this is at most one since $\frac{g_k}{c_k} \ge \frac{g_i}{c_i}$ for all $i\in[d]$. \checkmark

Suppose on the contrary that there exists another manipulated feature $\hat{\by} \neq \by^*$ that is optimal and is not of the form (\ref{Mx-def}):
\[ f(\hat{\by}) - (c(\hat{\by}) -c(\bx)) \ge 1 - \frac{c_k(g_0 - \bm{g}^\intercal \bx)}{g_k} \ge 0\]

Then it must be the case that moving to $\hat{\by}$ achieves a positive classification with a lower cost burden. We write 
\[
\hat{\by} =\bx + \sum_{i=1}\hat{t}_i\be_i
\] 
where $\be_i$ is the $i^{\text{th}}$ standard basis vector, and $\hat{t}_j = \hat{y}_j-x_j$ to highlight the components that have been manipulated from $\bx$ to $\hat{\by}$.

First, we suppose that $\hat{\by}$ is such that there exists some component $\hat{\by}_j > 0$ where $j \notin K = \argmax_{i\in [d]} \frac{g_i}{c_i}$. Now we construct a feature $\hat{\by}'$ by selecting this component, and decreasing $\hat{t}_j = 0$ and increasing a component $k \in K$ by $\frac{c_j \hat{t}_j}{c_k}$. That is 
\[ \hat{\by} ' = \hat{\by} - \hat{t}_j\be_j +\frac{c_j \hat{t}_j}{c_k} \be_k\]
The cost of manipulation from $\bx$ to $\hat{\by} '$ is the same as that for manipulation to $\hat{\by}$: 
\[c(\hat{\by}') - c(\bx) = \sum_{i=1}^d c_i \hat{y}_i - \hat{t}_j c_j + c_k \frac{c_j \hat{t}_j}{c_k} = \sum_{i=1}^d c_i \hat{y}_i \]
Notice that now we have 
\[
\sum_{i=1}^d g_i \hat{y}'_i = \sum_{i=1}^d g_i \hat{y}_i - g_j \hat{t}_j + \frac{g_k c_j \hat{t}_j}{c_k}
> \sum_{i=1}^d g_i \hat{y}_i \ge g_0.
\]
Thus the candidate can manipulate to $\hat{\by}'$ by expending the same cost with 
\[\sum_{i=1}^d g_i \hat{y}'_i > g_0\] 
Then by continuity of $g$, there must exist some $\bar{\by} \le \hat{\by}'$ such that $\sum_{i=1}^d g_i \bar{y}_i \in [g_0, \sum_{i=1}^d g_i \hat{y}'_i )$. Thus since costs are monotonically increasing, $c(\bx, \bar{\by}) < c(\bx, \hat{\by})$ and since $\bar{\by}$ reaches the same classification, and we have shown that $\hat{\by}$ could not have been optimal, which is a contradiction. \checkmark\\

Now we consider the case where $\hat{\by} = \bx + \sum_{i=1}^d {\hat{t}_i}\be_i$ is such that $\hat{t}_j = 0$ for all $j\notin K$, but $\bm{g}^\intercal \hat{\by} \neq g_0$. If $\bm{g}^\intercal \hat{\by} < g_0$, then $\hat{\by}$ is negatively classified and thus trivially receives a lower utility than manipulating to any feature $\by$ that is positively classified and associated with total cost $\sum_i t_i \le 1$.

If $\bm{g}^\intercal \hat{\by} > g_0$, then there are two possibilities: If $c(\hat{\by}) - c(\bx) \ge 1$, then once again, she receives at most a utility of 0, and thus manipulating to $\hat{\by}$ is a suboptimal move. If $c(\hat{\by}) - c(\bx) <1$, then we show the optimal manipulation is the one that moves from $\bx$ to 
\[\by = \bx + \sum_{i=1} t_i \be_i\]
where $\bm{g}^\intercal \by = g_0$ and $t_j = 0,$ $\forall j \notin K$---the move dictated by (\ref{Mx-def}). This feature $\by$ also achieves a positive classification, but we argue that it does so at a lower cost than $\hat{\by}$. Since $\bm{g}^\intercal \hat{\by} > g_0$, we can define
\[\Delta =\bm{g}^\intercal \hat{\by} - g_0 > 0 \] 
The manipulation from $\bx$ to $\hat{\by} - \frac{\Delta}{g_k}\be_k$ for any choice of $k$ attains a higher utility since it receives the same classification since 
\[\bm{g}^\intercal (\hat{\by} - \frac{\Delta}{c_k}\be_k) = g_0\]
but does so at a cost 
\[c({\by}) - c(\bx)= c(\hat{\by}) - c(\bx)- {\Delta}\]
Since we already showed that all manipulations to $\by$ of the form given in (\ref{Mx-def}) bear the same cost, then we have shown that all such $\by$ are preferable to $\hat{\by}$. By monotonicity of $c({\by}) - c(\bx)$ and $\sum_{i=1}^d g_i x_i$, all manipulations with lower cost entail a negative classification and thus a lower utility, and such only those manipulations to $\by$ are optimal.
\qed

\subsubsection{Proof of Theorem \ref{theorem:d_opt_learner}}
We first prove that a learner who has access to the linear decision boundary for the true classifier can construct a classifier that commits no errors on any candidates from a single group; thus, in our setting, perfect classifiers exist for groups $A$ and $B$. We then prove that all undominated classifiers commit no false positives on group $B$ and no false negatives on group $A$. 


Suppose true classifiers are given by $h_A$ and $h_B$ based on decision boundaries $\sum_{i=1}^d w_{A, i} x_i = \tau_A$ and $\sum_{i=1}^d w_{B, i} x_i = \tau_B$, costs are $c_A(\bx) = \sum_{i=1}^d c_{A, i} x_i$ and $c_B(\bx) = \sum_{i=1}^d c_{B, i} x_i$.\\


\textbf{Claim 1:} When facing candidates from a single group, a learner who has access to true decision boundary $\sum_{i=1}^d w_{i} x_i = \tau$ and manipulation costs $\sum_{i=1}^d c_i x_i$ can construct a perfect classifier. 
\begin{proof}
Consider those features $\bar{\bx} \in [0,1]^d$ that lie on the true decision boundary $\sum_{i=1}^d w_{i} x_i = \tau$ and thus have true labels $1$. For each of these $\bar{\bx}$, we construct $\Delta(\bar{\bx})$ as defined in (\ref{simplex-forward}) to represent the candidate's space of potential manipulation to form the set $\{\Delta(\bar{\bx})\}$ for all $\bar{\bx}$ on the boundary. Notice that when all candidates face the same cost, the set of $j^{\text{th}}$ vertices of each of the simplices $\Delta(\bar{\bx})$, given by $\bv_j(\bar{\bx}) = \bar{\bx} + \frac{1}{c_j}\be_j$, are coplanar. Each of these hyperplanes can be described as a set
\begin{equation*}
\label{outside-hyperplane}
\left\{ \by : \sum_{i=1}^d w_{i} y_i = \tau + \frac{w_j}{c_j}\right\}.
\end{equation*}
Let $k \in \argmax_j \frac{w_j}{c_j}$.  We define $g_1$ to be a notational shortcut for the hyperplane corresponding to feature $k$, so
\[
g_1 = \left\{ \by : \sum_{j=1}^d g_{1,j} y_j = g_{1,0} \right\},
\]
where $g_{1,0} = \tau + \frac{w_k}{c_k}$ and $g_{1,i} = w_i$ for all $i \in \{1,...,d\}$.
\ignore{
 so that
\begin{equation}
\label{g1-def}
\sum_{i=1}^d w_i y_i = \tau + \frac{w_k}{c_k} \iff  \sum_{j=1}^d g_{1,j} y_j = g_{1,0}.
\end{equation}
} 
We define a classifier $f_1$ based on the hyperplane $g_1$:
\begin{align}
\label{f1}
f_1(\by) = \begin{cases}
1 & \sum_{j=1}^d g_{1,j} y_j \ge g_{1,0} ,\\
0 & \sum_{j=1}^d g_{1,j} y_j < g_{1,0} .
\end{cases}
\end{align}

To show that $f_1$ is a perfect classifier of all candidates with these generic costs, we show that it commits no false positive errors and no false negative errors. Notice that since $g_1$ was constructed to be precisely the hyperplane that contains all vertices $\bv_k(\bar{\bx}) = \bar{\bx} + \frac{1}{c_k}$ of the simplices $\Delta(\bar{\bx})$ where $k \in \argmax_{j\in[d]} \frac{w_j}{c_j}$, then all $\bar{\bx}$ on the true decision boundary $\sum_{i=1}^d w_{i} x_i = \tau$ can indeed manipulate to $\bv_k(\bar{\bx})$ and reach $g_1$ to gain a positive classification.

Similarly, all candidates with features $\bx$ such that $\sum_{i=1}^d w_{i} x_i > \tau$, can move to the $k^{\text{th}}$ vertex of the simplex $\Delta(\bx)$ given by $\bv_k({\bx}) =  \bx + \frac{1}{c_k}\be_k$ in order to be classified positively since
\[ \sum_{i=1}^d w_{i}v_{k, i}({\bx}) > \tau + \frac{w_k}{c_k} \implies \sum_{i=1}^d g_{1, j} v_{k,i}({\bx})> g_{1,0} .\]
Thus $f_1$ correctly classifies all these candidates positively and permits no false negatives. \checkmark 

Consider the optimal manipulation for all true negative candidates $\bx$. By Lemma \ref{lemma:opt_candidate}, the optimal manipulation would be either to not move at all, guaranteeing a negative classification, or to move $\bx$ to some point $\by = \bx + \sum_{i=1}^d\frac{t_i}{c_i}\be_i$ where $t_j=0$ for all $j \notin \argmax_{j\in [d]}\frac{g_{1,j}}{c_j}$. 
But since $\sum_{i=1}^d w_{i} x_i < \tau$, then for all such $\by$,
\[ \sum_{i=1}^d w_{i} y_i  \leq \sum_{i=1}^d w_{i} x_i + \frac{w_k}{c_k} < \tau + \frac{w_k}{c_k} \implies \sum_{i=1}^d g_{1, j} y_j < g_{1,0}\]
and thus the classifier based on the hyperplane $g_1$ also issues a classification $f_1(\bx) = 0$ and admits no false positives.  \checkmark \\

Thus we have shown that the hyperplane $g_1$ supports a perfect classifier $f_1$ as defined in (\ref{f1}).
\end{proof}

Now we move on to group-specific claims, where groups have distinct costs and potentially distinct true decision boundaries, but we continue to use the constructions of $f_1$ and $g_1$ from Claim 1. \\

\textbf{Claim 2:}  Let $f^A_1$ be the classifier based on boundary $g_1$ for group $A$, and let $f^B_1$ be the classifier based on boundary $g_1$ for group $B$, as in (\ref{f1}), but with group-specific costs and true decision boundary parameters. Then $\forall \by \in [0,1]^d$, \[f^A_1(\by) =1 \implies f^B_1(\by) = 1.\]
\begin{proof}
We first prove the claim for the case in which $h_A = h_B$ with decision bounday $\sum_{i=1}^d w_i x_i = \tau$. We then show that it also holds when the two are not equal.

By the cost condition that $c_A(\by) - c_A(\bx) \le c_B(\by) - c_B(\bx)$ for all $\bx \in [0,1]^d$ and $\by \ge \bx$, we know that for any given $\bx$, \[\Delta_B(\bx) \subseteq \Delta_A(\bx).\]


Let $k_A \in \argmax_{j\in [d]} \frac{w_j}{c_{A,j}}$ and $k_B \in \argmax_{j\in [d]} \frac{w_j}{c_{B,j}}$, so that $g^A_1$ and $g^B_1$ are defined as
\[\sum_{i=1}^d w_i y_i = \tau + \frac{w_{k_A}}{c_{A,k_A}} \iff g^A_1: \sum_{j=1}^d g^A_{1,j} y_j = g^A_{1,0}, \]
\[\sum_{i=1}^d w_i y_i = \tau + \frac{w_{k_B}}{c_{B, k_B}} \iff g^B_1: \sum_{j=1}^d g^B_{1,j} y_j = g^B_{1,0} .\]
Then since for all $i \in [d]$, $c_{A, i} \le c_{B, i}$, we must have that 
\[\tau + \frac{w_{k_A}}{c_{A,k_A}}  \ge \tau + \frac{w_{k_B}}{c_{A,k_B}} \ge \tau + \frac{w_{k_B}}{c_{B, k_B}},\] and thus $g^A_{1, 0} \ge g^B_{1,0}$. Since $f^A_1$ is the classifier based on $g^A_1$ and $f^B_1$ is based on $g^B_1$, we have that $\forall \by \in [0,1]^d$, \[f^A_1(\by) = 1 \implies  f^B_1(\by) = 1.\]

Now consider the case in which $h_A$ and $h_B$ differ. Recall the assumption $h_A(\bx) = 1 \implies h_B(\bx) = 1$ for all $\bx \in [0,1]^d$. Thus for all $\bx \in[0,1]^d$, 
\begin{equation}
\label{condition-ordering}
\sum_{i=1} w_{A, i} x_i \ge \tau_A \implies \sum_{i=1} w_{B, i}x_i\ge \tau_B.
\end{equation}
Recall that the hyperplanes ${g}_1^A$, ${g}_1^B$ are constructed as shifts of $\sum_{i=1} w_{A, i} x_i \ge \tau_A$ and  $\sum_{i=1} w_{B, i}x_i\ge \tau_B$ by the set of simplices $\{\Delta_A(\bar{\bx}_A)\}$ and $\{\Delta_B(\bar{\bx}_B)\}$ for $\bar{\bx}_A$ such that $\sum_{i=1} w_{A, i} \bar{x}_{A,i} = \tau_A $ and $\bar{\bx}_B$ such that $\sum_{i=1} w_{B, i} \bar{x}_{B,i} = \tau_B$. Since $\Delta_B(\bx) \subseteq \Delta_A(\bx)$, ${g}_1^A$ and ${g}_1^B$ support classifiers $f_1^A$ and $f_1^B$ such that \[f_1^A(\by) = 1\implies f_1^B(\by) = 1.\]
\end{proof}

\textbf{Claim 3:} All undominated classifiers commit no false negative errors on group $A$ members and no false positive errors on group $B$ members when candidates best respond. 
\begin{proof}
Fix a classifier $f$ and consider a group $A$ candidate with true feature vector $\bar{\bx}$ who manipulates to best response $\bar{\by}$ such that $h_A(\bar{\bx}) = 1$ but $f(\bar{\by}) = 0$.  Thus the classifier $f$ makes a false negative error on this candidate. We show that we can construct another classifier $\hat{f}$ that correctly classifies $\bar{\bx}$ under its optimal manipulation with respect to $\hat{f}$.

We prove that $\hat{f}$ commits no more errors than does $f$ and commits strictly fewer errors since it commits no false negatives on group $A$ candidates.

Construct the classifier $\hat{f}$ such that
\begin{align}
\label{undominated-b}
\hat{f}(\by) = \begin{cases}
1 &  f(\by) = 1 \text{ or } f^A_1(\by) = 1, \\
0 & \text{ otherwise,}
\end{cases}
\end{align}
where $f^A_1(\by)$ is based on the boundary $\sum_{j=1}g^A_{1,j}y_j = g^A_{1,0}$.


We first argue that $f$ and $\hat{f}$ make exactly the same set of false positive errors.  

Consider a potential false positive error that $\hat{f}$ issues on a candidate with feature $\bx$ from group $A$. Such a candidate cannot manipulate to a feature $\by$ to ``trick'' classifier $f_1^A$, since we have shown in Claim 1 that $f_1^A$ perfectly classifies all group $A$ candidates, and thus does not admit false positives. Thus any potential false positive error must be due to $f(\by) = 1$, in which case $\hat{f}$ and $f$ issue the same false positive error.

Now we consider a potential false positive error that $\hat{f}$ issues on a candidate with feature $\bx$ from group $B$. By Claim 2, $f_1^A(\by) = 1 \implies f_1^B(\by) = 1$, and thus we would have that the candidate with feature $\bx$ was able to manipulate to some feature $\by$ such that $f_1^B(\by) = 1$. But this is a contradiction, since we know that $f_1^B$ commits no false positives on group $B$ members, and thus $f_1^A(\by)$ does not commit false positives on group $B$. Thus if $\hat{f}$ commits a false positive, then it must be the case that $f$ committed the same false positive. 

Consider a potential false negative error that $\hat{f}$ issues on a candidate with feature $\bx$ from group $B$. Then it must be the case that $\bx$ can manipulate to some $\by$ such that \emph{both} $f(\by) = 0$ and $f_1^A(\by)=0$, and thus it be the case that ${f}$ commits the same false negative.

Lastly, consider a potential false negative error on a candidate from group $A$. By claim 1, this candidate must have been able to manipulate to some feature vector $\by$ such that $f_1^A(\by) = 1$, since $f_1^A$ commits no errors on group $A$ members. Thus when a candidate with unmanipulated feature $\bx$ can manipulate to some $\by$ such that $f_1^A(\by) = 1$ yet can only present a (possibly different) feature $\by$ such that $f(\by) =0$, then $\hat{f}$ correctly classifies this candidate positively, even when $f$ does not. Thus $\hat{f}$ makes no false negative errors on group $B$. 

\ignore{
Now consider the group $A$ candidate that is a true positive but is incorrectly classified $f(\by)=0$. As we have already shown in Claim 1, $f_1^A(\by)$ commits no errors of either type on candidates from group $A$. Thus if a candidate from group $A$ presenting $\by$ is a genuine positive, she will be classified correctly by $f_1^A$ such that $f_1^A(\by)=1$. As a result, $\by$ is classified positively by $\hat{f}$, correcting the false negative error issued by $f$. $\hat{f}$ does not commit any more false negatives on group $B$ than does $f$ since it will only classify a candidate negatively if both $f(\by) = 0$ and $f^A_1(\by) =0$. 
}

Thus $\hat{f}$ commits strictly fewer errors than $f$---none of which are false negatives on group $A$ members---and $f$ is dominated by $\hat{f}$. \checkmark \\

The second half of the claim can be proved through an analogous argument.  
\end{proof}



Combining Claims 1 and 3, we conclude that we can construct perfect classifiers for group $A$ that commit only false negative errors on group $B$ and perfect classifiers for group $B$ that commit only false positive errors on group $A$. $f_1^A$ and $f_1^B$ are examples of such classifiers, though they are not unique.  \\

\qed

\subsubsection{Proof of Lemma~\ref{linear-d}}
$\implies$ direction: Assume a group $m$ candidate with feature $\bx$ can move to $\by$ such that $f(\by) = 1$ and $c_m(\by) - c_m(\bx) \le 1$, we show that necessarily $\bx \ge \ell$ for some $\ell \in \mathcal{L}_m(g)$. \\

If $\bx$ can move to $\by$, then $\bx \in \Delta^{-1}(\by)$. By the definition of $\ell_m(\by), \bx \ge \bar{\bx}$ for some $\bar{\bx} \in \ell_m(\by)$. Then by monotonicity of $g$, we have that 
\[\sum_{i=1}^d g_i x_i \ge \sum_{i=1}^dg_i \bar{x}_i \ge  \min_{x\in\ell_m(\by)} \sum_{i=1}^dg_i {x}_i \]
Thus $\bx \ge \ell$ for some $\ell \in \mathcal{L}_m(g)$. \checkmark \\

$\impliedby$ direction: Assume some group $m$ candidate has feature $\bx \ge \ell$ for some $\ell \in \mathcal{L}_m(g)$. Then she can move to some $\by$ such that $f(\by) = 1$ and $c_m(\by) - c_m(\bx) \le 1$.\\

If $\bx\ge \ell$ for some $\ell \in \mathcal{L}_m(g)$, then 
\[\sum_{i=1}^d g_i x_i \ge \sum_{i=1}^d g_i \ell_i ,\]
where $\ell \in \Delta^{-1}_m(\by)$ for some $\by$ such that $\sum_{i=1}g_iy_i = g_0$ and $f(\by) = 1$. Since $\ell$ is defined as $\argmin_{x \in\ell_m(\by)}\sum_{i=1}g_ix_i$, then we have
\[\sum_{i=1}^d g_i \ell_i = \sum_{i=1}^d \big(g_i y_i - \max_{t_i} \sum_{i=1}^d\frac{g_i t_i}{c_{m,i}}\big) ,\]
where $t_i \ge 0$ and $\sum_{i=1}t_i = 1$ as shown before. Then substituting $\sum_{i=1}^d g_i y_i = g_0$, we have that
\[\sum_{i=1} g_i \ell_i + \frac{g_{k_m}}{c_{m,k_m}} = g_0 ,\]
where $ k_m \in \argmax_{i=[d]} \frac{g_{i}}{c_i}$. Since $\bx \ge \ell$, $\bx$ can also manipulate to some $\by$ with $f(\by) = 1$, bearing a cost $\le 1$. 
\qed

\subsubsection{Proof of Proposition \ref{cost-d-dim}}

If a learner publishes an undominated classifier $f$, then by Theorem \ref{theorem:d_opt_learner}, the hyperplane $g: \bm{g}^\intercal \bx = g_0$  that supports this classifier can only commit inequality-reinforcing errors: only false positives on group $A$ members and only false negatives on group $B$ members. 

As proved in Lemma \ref{linear-d}, the set $\mathcal{L}_m(g)$ determines the effective threshold on unmanipulated features $\bx$ for a candidate of group $m$. We have already shown that for any two $\ell_1, \ell_2 \in \mathcal{L}_m(g)$, \[\sum_{i=1}g_i \ell_{1,i} =\sum_{i=1}g_i \ell_{2,i} = g_0 -\frac{g_{k_m}}{c_{k_m}}\] where $ k_m \in \argmax_{i=[d]} \frac{g_{i}}{c_{m,i}}$. For any $\ell \in \mathcal{L}_B(g)$, we have
\[\sum_{i=1}^dg_i\ell_i + \frac{g_{k_B}}{c_{B, k_B}}= g_0\]
Thus combining these results, those group $B$ candidates with features $\bx \in [0,1]^d$ in the intersection 
\[ \bm{g}^\intercal \bx < g_0 - \frac{g_{k_B}}{c_{B, k_B}} \bigcap \bw_B^\intercal \bx \ge \tau_B  \]
are classified as false negatives. For group $A$, we consider $\ell \in \mathcal{L}_A(g)$:
\[\sum_{i=1}^dg_i\ell_i + \frac{g_{k_A}}{c_{A, k_A}}= g_0\]
and thus group $A$ candidates with features $\bx \in [0,1]^d$ in the intersection 
\[\bw_A^\intercal \bx < \tau_A  \bigcap  \bm{g}^\intercal \bx \ge g_0 - \frac{g_{k_A}}{c_{A, k_A}} \]
are classified as false positives. Thus the cost publishing $g$ is
\begin{align*}
&C_{FN} P_{x \sim \mathcal{D}_B}\big[ \bx \in \big(\bm{g}^\intercal \bx < g_0 - \frac{g_{k_B}}{c_{k_B}} \bigcap \bw_B^\intercal \bx \ge \tau_{B}\ \big) \big] \\
&+ C_{FP} P_{x \sim \mathcal{D}_A}\big[\bx \in \big(\bw_A^\intercal \bx < \tau_{A} \bigcap\bm{g}^\intercal \bx \ge g_0  - \frac{g_{k_A}}{c_{k_A}} \big) \big] 
\end{align*}
\qed

\subsection{Proofs from Section~\ref{sec:subs}}

\subsubsection{Reduction from the $d$-dimensional setting to the one-dimensional setting} We first show that under certain conditions of a learner's equilibrium classifier strategy, a $d$-dimensional subsidy analysis is equivalent to a one-dimensional subsidy analysis.
\label{reduction}

In general $d$-dimensions, those features $\by$ attainable from an unmanipulated feature $\bx \in [0,1]^d$, where $f(\bx) = 0$, is given by
\[ \by \le \bx + \sum_{i=1}^d \frac{t_i}{c_i}\be_i  \text{ where } \sum_{i=1}^d t_i = 1\]
where the right hand side gives the simplex $\Delta(\bx)$ of potential manipulation. By Lemma \ref{lemma:opt_candidate}, if a candidate moves from $\bx$ to $\by \neq \bx$, then she selects $\bm{t}$ such that $t_j=0$ for all $j\notin K = \argmax_{i=[d]}\frac{g_i}{c_i}$. Staying within the simplex implies $\sum_{i=1}^d t_i \le 1$.

Increasing the candidate's available cost to expend from $1$ to $n$ increases her range of motion such that now she can move to any \[ \by \le \bx + \sum_{i=1}^d \frac{t_i}{c_i}\be_i  \text{ where } \sum_{i=1}^d t_i = n\] She continues to manipulate in the spirit of Lemma \ref{lemma:opt_candidate}---optimal moves entail choices of $\bm{t}$ such that $t_j = 0$ for all $j\notin K$---however now, she is willing to manipulate if $\exists i \in [d]$ such that 
\[f(\bx + \frac{n}{c_i}\be_i) = 1\]
and thus chooses $\bm{t}$ such that $\sum_{i=1}t_i \le n$. 

Since offering a subsidy does not change the form of the group $B$ cost function, a candidate from group $B$ will pursue the same manipulation strategy given by the vector $\bm{t}$ under subsidy regimes as long as the classifier's decision boundaries stay the same. By definition, all such choices of $\by$ resulting from a manipulation via $\bm{t}$ have equivalent values $\bm{g}^\intercal \by$. 

When costs are subsidized through a flat $\alpha$ or a proportional $\beta$ subsidy, a candidate with feature $\bx$ can manipulate to any $\by_\alpha, \by_\beta \ge \bx$ that satisfies
\begin{align}
\label{max-alpha-d}
\by_\alpha &\in [\bx, \bx + \sum_{i=1}^d \frac{t_i}{c_i}\be_i]  \text{ where } \sum_{i=1}^d t_i = 1 + \alpha \\
\label{max-beta-d}
\by_\beta &\in [\bx, \bx + \sum_{i=1}^d \frac{t_i}{c_i}\be_i]  \text{ where } \sum_{i=1}^d t_i = \frac{1}{\beta}
\end{align}

We can pursue a dimensionality reduction by mapping each feature $\bx \in [0,1]^d$ to $\bm{g}^\intercal\bx \in \mathbb{R}_+$. Rather than considering an optimal manipulation in $d$-dimensions from $\bx$ to $\by$, we instead consider the relationship between the cost of the manipulation and the change from $\bm{g}^\intercal \bx$ to $\bm{g}^\intercal\by$:
\[\sum_{i=1}^d c_i (y_i - x_i) \iff \sum_{i=1}^d {g_i}(y_i-x_i) \]
where $g_i$ gives the coefficients of the linear decision boundary that supports $f$, and $\bx$ optimally manipulates to $\by$. We want to show that such a relationship is linear.

Consider optimal manipulations: If a candidate chooses not to manipulate at all, she will incur a cost of $0$ and will also move from $\sum_{i=1}^d g_i (y_i - x_i) = 0$. Since optimal manipulations (under any ``budget" constraint) only are along $k^{\text{th}}$ components, a move from $\bx$ to $\by$ always entails a total cost of
\[\sum_{i\in K}^d c_i(y_i - x_i) \]
accompanied with 
\[ \sum_{i\in K}^d g_i(y_i - x_i)=\bm{g}^\intercal (\by - \bx) \]
Thus we can write her total cost $c$ for a move from $\bx$ to $\by$ as 
\begin{equation}
\label{one-d-cost}
\frac{c_k}{g_k}(\bm{g}^\intercal \by - \bm{g}^\intercal\bx) 
\end{equation}
for any $k \in K$. Recall that by Lemma \ref{lemma:opt_candidate}, optimal non-stationary manipulations move from $\bx$ to $\by>\bx$ such that $\sum_{i=1}^dg_iy_i= g_0$, so in these cases, we can also write the above as
\[\frac{c_k}{g_k} (g_0 - \bm{g}^\intercal\bx) \]
Thus we can consider candidates' unmanipulated $d$-dimensional features $\bx$ as one-dimensional features $\bm{g}^\intercal \bx$ and classifiers $f$ based on $d$-dimensional hyperplanes $g: \sum_{i=1}^d g_ix_i = g_0$ as imposing one-dimensional thresholds $g_0$. 

However a learner may also choose a different optimal subsidy strategy, thus publishing a classifier that now admits candidates differently. Formally, suppose a learner first publishes a classifier $f_1$ based on a decision boundary $g_1: \sum_{i=1}^d g_{1,i}x_i = g_{1,0}$ to which a candidate's optimal response follows the form given in Lemma \ref{lemma:opt_candidate} with $k_1 \in \argmax_{i\in [d]} \frac{g_{1,i}}{c_i}$. If a learner then chooses to change her strategy when implementing a subsidy, thus publishing a different classifier $f_2$ based on decision boundary $g_2: \sum_{i=1}^d g_{2,i} x_i = g_{2,0}$, a candidate's optimal manipulation strategy will continue to adhere to Lemma \ref{lemma:opt_candidate}, however, now, $k_2 \in \argmax_{i\in [d]} \frac{g_{2,i}}{c_i}$. Whereas the corresponding one-dimensional cost function $c(\by) - c(\bx)$ for best-response manipulations when facing classifier $f_1$ was given by 
\[\frac{c_{k_1}}{g_{1,k_1}} ( \bm{g}_1^\intercal(\by - \bx)) \]
Her corresponding cost function when facing classifier $f_2$ is
\[\frac{c_{k_2}}{g_{2,k_2}} (\bm{g}_2^\intercal(\by-\bx)) \]
When these cost functions are the same, as when the coefficients $g_{1,i} = g_{2,i}$ for all $i$, the agent's strategies when facing $f_1$ and $f_2$ are identical when reduced to one-dimension. This case arises, for example, when the learner continues to perfectly classify a single group in both the non-subsidy regime and the subsidy regime. In these cases, we can transition to considering just one-dimensional manipulations from $\bm{g}^\intercal \by$ to $\bm{g}^\intercal\bx$, where candidates bear linear costs of manipulation given in (\ref{one-d-cost}).\\


%

\subsubsection{Proof of Proposition \ref{surprising2}}
Working from the subsidy and no-subsidy comparisons given in Proposition \ref{surprising1}, we show that all three parties would have preferred the outcomes of a non-manipulation world to those in both of the manipulation cases. 

To facilitate comparisons of welfare across classification regimes, we formalize group-wide utilities in the following definition.

\begin{definition}[Group welfare under a proportional subsidy]
The average welfare of group $B$ under classifier $f_{prop}$ and a proportional subsidy with parameter $\beta$ is given by
\begin{equation*}
\begin{split}
W_B(f_{prop}, \beta) = & \int_{R_1} P_{x\sim \mathcal{D}_B}(x)dx\\
&+ \int_{R_2} \big(1 -  \beta(c_B(y(x)) - c_B(x)) \big) P_{x\sim \mathcal{D}_B}(x)dx,
\end{split}
\end{equation*}
\begin{align*}
W_A(f_{prop}, 1) = & \int_{R_1} P_{x\sim \mathcal{D}_A}(x)dx\\
&+ \int_{R_2} \big(1 - (c_A(y(x)) - c_A(x)) \big) P_{x\sim \mathcal{D}_A}(x)dx,
\end{align*}
where $y(x)$ is the best response of a candidate with unmanipulated feature $x$, $R_1$ sums over those candidates who are positively classified by $f_{prop}$ without expending any cost, and $R_2$ sums over those candidates who are positively classified after manipulating their features. Since group $A$ members do not receive subsidy benefits, their welfare form is the same across no-subsidy and subsidy regimes.

We use $W_A(f_{prop})$ to denote $W_A(f_{prop}, 1)$, the average welfare for group $A$ under classifier $f_{prop}$ with no subsidy.
\end{definition}

\begin{definition}[Group welfare in a non-manipulation setting]
The average welfare of group $m$ under classifier $f_0$ in a non-manipulation setting is given by
\begin{align*}
W_{m}(f_0) = & \int_{R} P_{x\sim \mathcal{D}_m}(x)dx
\end{align*}
where $R$ sums over candidates who are positively classified by $f_0$.
\end{definition}

\begin{proposition}
\label{surprising2-formal}
There exist cost functions $c_A$ and $c_B$ satisfying the cost conditions, learner distributions $\mathcal{D}_A$ and $\mathcal{D}_B$, true classifiers with threshold $\tau_A$ and $\tau_B$, population proportions $p_A$ and $p_B$, and learner penalty parameters $C_{FN}$, $C_{FP}$, and $\lambda$, such that
\[W_A (f_{prop}^*) < W_A (f_0^*), \qquad W_B(f_{prop}^*, \beta^*) < W_B (f_0^*),\]
\[W_A (f_1^*) < W_A (f_0^*), \qquad W_B(f_1^*) < W_B (f_0^*),\]
\[C(f_{prop}^*, \beta^*) > C(f_0^*), \qquad C(f_1^*) > C(f_0^*)\]
where $f_0$ is the equilibrium classifier in the non-manipulation regime, $f_1^*$ is the equilibrium classifier in the manipulation regime, and ($f_{prop}^*, \beta^*$) is the equilibrium classifier in the subsidy regime. The average welfare of each group, $W_m(\cdot)$, as well as the learner, $1-C(\cdot)$, is higher at the equilibrium of the non-manipulation game compared with the equilibria of the Strategic Classification Game with proportional subsidies and compared with the equilibrium of the Strategic Classification Game with no subsidies.
\end{proposition}

\begin{example} 
\label{example-linear}
Now we consider a case in which candidates have linear cost functions
$c_A(x) = 3x$ and $c_B(x) = 4x$.
To show that diminished welfare for both candidate groups can occur without requiring distortions of probability distributions or cost functions, we consider a learner who seeks to avoid errors on group $B$ in both the subsidy and the non-subsidy regimes by penalizing false negatives twice as much as false positives, with
$C_{FN} = \frac{2}{3}$, $C_{FP} = \frac{1}{3}$, and $\lambda = \frac{3}{4}$.  As in the previous example, we assume that the underlying unmanipulated features for both groups are uniformly distributed with $p_A = p_B = \frac{1}{2}$, and that $\tau_A = 0.4$ and $\tau_B = 0.3$.

Now the equilibrium learner classifier without subsidies is based on threshold
$\sigma_1^* = \sigma_B = 0.55$,
which perfectly classifies all candidates from group $B$, while permitting false positives on candidates from group $A$ with features
$x \in [0.217, 0.4)$.
Under a proportional subsidy intervention, the learner's equilibrium action is to choose threshold
$\sigma_{prop}^* = \sigma^\beta_B \approx 0.552$ and $\beta^* = 0.994$,
which again perfectly classifies $B$ candidates. Notice that now her optimal threshold commits fewer false positive errors on group $A$ members, while still committing false positives on those members with features
$x \in [0.219, 0.4)$.

Here, even when the learner has a cost penalty that is explicitly concerned with mistakenly excluding group $B$ candidates and then seeks to offer a subsidy benefit to further alleviate their costs, group $B$ members are still no better off. They receive the same classifications as before and it can be shown that all candidates who manipulate must spend more to reach the higher threshold, even while accounting for the subsidy benefit! 
Some group $A$ candidates are also worse off since the threshold has increased, and they receive no subsidy benefits. 
As before, only the learner gains from the intervention.
\end{example}

\begin{example}
This example is based on Example \ref{example-linear}. Now we consider the case a learner seeks $\sigma_1^* \in [\sigma_B, \sigma_A]$ where $\sigma_A = 0.733$ and $\sigma_B = 0.55$. Suppose she seeks to equalize the number of false positives she commits on group $A$ and the number of false negatives for group $B$ and thus chooses $\sigma_1^* = 0.64$ such that
\[\ell_A(\sigma_1^*) = 0.31\]
\[\ell_B(\sigma_1^*) = 0.39\]
Thus group $B$ candidates with features $x \in [0.3, 0.39)$ are mistakenly excluded, and group $A$ candidates with features $x\in [0.31, 0.4)$ are mistakenly admitted. 

Upon implementing a subsidy and minimizing the same error penalty as in Example 1, the learner selects an optimal proportional $\beta$ subsidy such that
\[\sigma_{prop}^* = \sigma_A = 0.733; \beta = 0.806\]
Under this regime, group $B$ members are worse-off because many more candidates now receive false negative classifications
\[x \in [0.3, 0.423) \]
Others who do secure positive classifications must pay more to do so. Candidates in group $A$ are now perfectly classified, though this actually entails a welfare decline, since some candidates lose their false positive benefits. The learner is also strictly better off with a total penalty decline 
\[C(\sigma_0^*) = 0.183 \rightarrow C(\sigma_{prop}^*, \beta^*) = 0.128\]
Recall that the learner's utility is given by $1 - C(\cdot)$. Thus we have that
 \[W_A(\sigma_{prop}^*, \beta^*) < W_A(\sigma_1^*)\] 
  \[W_B(\sigma_{prop}^*, \beta^*) < W_B(\sigma_1^*)\] 
   \[C(\sigma_1^*) > C(\sigma_{prop}^*, \beta^*)\] 
  Now consider a non-manipulation regime, in which the learner selects to equalize the number of false negatives for group $B$ and the number of false positives for group $A$, she now chooses a threshold on unmanipulated features
  \[\sigma_0^* = 0.35\]
  Some group $A$ candidates lose false positive benefits in the manipulation regime, though on the whole, the group fares better off because all those candidates with features
  \[x \in [0.39, 0.64)\]
  need not expend any costs in order to receive a positive classification.
  Group $B$ candidates are strictly better off since they both receive fewer false negatives and need not pay to manipulate. The learner is also better off here because she reduces her error down to $C(\tau^*) = 0.1$. Thus comparing the non-manipulation regime, the no-subsidy manipulation regime, and the subsidy regime, we have that utility comparisons for all three parties is given by
   \[W_A(\sigma_0^*)>W_A(\sigma_1^*)>W_A(\sigma_{prop}^*, \beta^*)\] 
  \[W_B(\sigma_0^*)>W_B(\sigma_1^*)>W_B(\sigma_{prop}^*, \beta^*) \] 
   \[ 1-C(\sigma_{0}^*)>1-C(\sigma_{prop}^*, \beta^*) > 1- C(\sigma_{1}^*)  \]
\end{example}


%

\subsection{Flat Subsidies}

Here we give analogous definitions and results for flat subsidies in which the learner absorbs up to a flat $\alpha$ amount from each group $B$ candidate's costs and show that qualitatively similar results hold.

\begin{definition}[Flat subsidy] 
Under a flat subsidy plan, the learner pays an $\alpha > 0$ benefit to all members of group $B$. As such, a group B candidate who manipulates from an initial score $\bx$ to a final score $\by \ge \bx$ bears a cost of $\max\{0, c_B(\by) - c_B(\bx) - \alpha\}$.
\end{definition}

A learner's strategy now consists of both a choice of $\alpha$ and a choice of classifier $f$ to issue. The learner's goal is to minimize her penalty
\begin{equation*}
\begin{split}
&C_{FP} \sum_{m \in \{A, B\}} p_m P_{\bx\sim\mathcal{D}_m}[h_m(\bx) = 0, f(\by) = 1] \\
& + C_{FN}\sum_{m \in \{A, B\}} p_m P_{\bx\sim\mathcal{D}_m}[h_m(\bx) = 1, f(\by) = 0] \\
& + \lambda cost(f, \alpha),
\end{split}
\end{equation*}

We can define
\[
\ell_B^\alpha(y) = c_B^{-1}\Big(c_B(y) - (1+\alpha)\Big).
\]
Under the $\alpha$ subsidy, for an observed feature $y$, the group $B$ candidate must have unmanipulated feature $x \ge \ell^\alpha_B(y)$.

From these functions, we define $\sigma_B^\alpha$ and $\sigma_B^\beta$ such that $\ell_B^\alpha(\sigma_B^\alpha) = \tau_B$, and $\ell_{B}^\beta(\sigma_B^\beta) = \tau_B$. Under a flat $\alpha$ subsidy, setting a threshold at $\sigma^\alpha_B$ correctly classifies all group $B$ members; under a proportional $\beta$ subsidy, a threshold at $\sigma^\beta_B$ correctly classifies all group $B$ members.

From this, we define $\sigma_B^\alpha$ such that $\ell_{B}^\alpha(\sigma_B^\alpha) = \tau_B$. Under a flat $\alpha$ subsidy, setting a threshold at $\sigma^\alpha_B$ correctly classifies all group $B$ members.

In order to compute the cost of a subsidy plan, we must determine the number of group $B$ candidates who will take advantage of a given subsidy benefit. Since manipulation brings no benefit in itself, candidates will still only choose to manipulate and use the subsidy if it will lead to a positive classification. For the flat $\alpha$ subsidy, $cost(f, \alpha)$ is given by
\begin{equation*}
\int_{c_B^{-1}(c_B(\sigma) - \alpha)}^\sigma
\!\!\!\!\!\!\!\![c_B(\sigma) - c_B(x)] P_{D_B}(x) dx + \alpha \int_{\ell_B^\alpha(\sigma)}^{c_B^{-1}(c_B(\sigma) - \alpha)} P_{D_B}(x)dx,
\end{equation*}
where $\sigma$ is the threshold for classifier $f$. The first integral refers to the benefits paid out to candidates with manipulation costs less than the $\alpha$ amount offered. The latter refers to the total sum of full $\alpha$ payments offered to those with costs greater than $\alpha$. 

\begin{definition}[Group welfare under a flat subsidy]
The average welfare of group B under classifier $f$ and a flat subsidy with parameter $\alpha$ is given by 
\begin{align*}
W_B(f, \alpha) = &\int_{R_1} P_{x\sim\mathcal{D}_B}(x) dx \\
&+ \int_{R_2} (1-c_B(y(x) - c_B(x))) P_{x\sim\mathcal{D}_B}(x) dx 
\end{align*}
where $y(x)$ is the best response of a candidate with unmanipulated feature $x$, $R_1$ sums over those candidates who are positively classified without expending any cost, and $R_2$ sums over those candidates who are positively classified after manipulating their features. Note that under the flat subsidy, group $B$ costs have the form $\max\{0,c_B(y) - c_B(x) - \alpha\}$ The formulation of average group $A$ welfare is the same in this setting and follows the same form given in Definition 5.
\end{definition}

\begin{theorem}[Subsidies can harm both groups]
\label{surprising1-alpha}
There exist cost functions $c_A$ and $c_B$ satisfying the cost conditions, learner distributions $\mathcal{D}_A$ and $\mathcal{D}_B$, true classifiers with threshold $\tau_A$ and $\tau_B$, population proportions $p_A$ and $p_B$, and learner penalty parameters $C_{FN}$, $C_{FP}$, and $\lambda$, such that
\[W_A (f_{prop}^*) < W_A (f_0^*), \qquad W_B(f_{prop}^*, \alpha^*) < W_B (f_0^*),\]
where $f_{prop}^*$ and $\alpha^*$ are the learner's equilibrium classifier and subsidy choice in the Strategic Classification Game with flat subsidies and $f_0^*$ is the learner's equilibrium classifier in the Strategic Classification Game with no subsidies.
\end{theorem}

\end{document}